\documentclass{uai2021} 

\usepackage[american]{babel}

\usepackage{natbib} 
\usepackage{mathtools} 
\usepackage{booktabs} 
\usepackage{tikz} 
\usepackage{hyperref}

\usepackage{pifont}
\newcommand{\cmark}{\ding{51}}%
\newcommand{\xmark}{\ding{55}}%

\usepackage{amsthm}
\usepackage{booktabs}
\usepackage{algorithm}
\usepackage{algorithmic}
\usepackage[capitalise]{cleveref}

\usepackage{graphicx}
\usepackage{subfigure} 
\usepackage{bm}
\usepackage{amssymb}

\newcommand{\St}{{\mathcal S}} 
\newcommand{\Ac}{{\mathcal A}} 
\newcommand{\T}{{P}} 
\newcommand{\R}{{r}} 

\newcommand{\id}{{\bm \mu}} 
\newcommand{\sd}{\bm d} 

\newcommand{\J}{J} 

\newcommand{\vR}{{\bm r}} 
\newcommand{\vpp}{{\bm \theta}} 
\newcommand{\Cost}{{c}} 
\newcommand{\vCb}{{\bm d}} 
\newcommand{\Cb}{{d}} 
\newcommand{\nC}{{m}} 
\newcommand{\V}{{V}} 
\newcommand{\Q}{{Q}} 
\newcommand{\A}{A} 
\newcommand{\Z}{Z} 
\newcommand{\Zf}{{\bm \Z}} 
\newcommand{\Y}{Y} 
\newcommand{\Yf}{{\bm \Y}} 
\newcommand{\NN}{{\Psi}} 
\newcommand{\Cf}{{\rho}} 

\usepackage{etoolbox}
\newtoggle{final}
\toggletrue{final}

\newcommand{\jianyi}[1]{\iftoggle{final}{#1}{{\color{blue} #1}}}
\newcommand{\paul}[1]{\iftoggle{final}{#1}{{\color{orange} #1}}}

\newtheorem{theorem}{Theorem}

\title{Safe Distributional Reinforcement Learning}
\author[1]{\href{mailto:Jianyi Zhang <zhangjy97@sjtu.edu.cn>?Subject=Safe Distributional Reinforcement Learning}{Jainyi~Zhang}{}}
\author[1,2]{Paul Weng}
\affil[1]{%
    UM-SJTU Joint Institute\\
    Shanghai Jiao Tong University\\
    Shanghai, China
}
\affil[2]{%
    Department of Automation\\
    Shanghai Jiao Tong University\\
    Shanghai, China
}

\begin{document}
\maketitle

\begin{abstract}
    Safety in reinforcement learning (RL) is a key property in both training and execution in many domains such as autonomous driving or finance.
    In this paper, we formalize it with a constrained RL formulation in the distributional RL setting.
    Our general model accepts various definitions of safety (e.g., bounds on expected performance, CVaR, variance, or probability of reaching bad states).
    To ensure safety during learning, we extend a safe policy optimization method to solve our problem.
    The distributional RL perspective leads to a more efficient algorithm while additionally catering for natural safe constraints. 
    We empirically validate our propositions on artificial and real domains against appropriate state-of-the-art safe RL algorithms.
\end{abstract}

\section{Introduction}\label{sec:intro}

Reinforcement learning (RL) has shown great promise in various applications \citep{
SilverSchrittwieserSimonyanAntonoglouHuangGuezHubertBakerLaiBoltonChenLillicrapHuiSifreDriesscheGraepelHassabis17,jinDeepLearningAlibaba2017}.
As such techniques start to be deployed in real applications, safety in RL \citep{GarciaFernandez15} starts to be recognized as a key consideration both during learning, but also during execution after training.
Indeed, in many domains from medical applications to autonomous driving to finance, the actions chosen by an RL agent can have disastrous consequences and therefore the corresponding risks need to be controlled both during training, but also during execution.

While traditional RL does not take safety into account, recent work has started to studied it more actively.
Safety takes various definitions in the literature.
In its simplest sense, it means avoiding bad states \citep{geibel2005risk}, but it can take more general meaning such as decision-theoretic risk aversion \citep{Borkar10}, or risk constraints \citep{prashanth2016variance}, satisfaction of logic specifications \citep{alshiekh2017safe}, but also simple bounds on expected cumulated costs \citep{yu2019convergent}.

For a given definition of safety, one may want to learn a policy that satisfies it, without constraining the training process. 
Such approach would provide a safe policy for deployment after training.
In contrast, recent work in safe RL aims at enforcing safety during learning as well, which is a difficult task as the RL agent needs to explore.



This paper follows this latter trend and safety is formulated as the satisfaction of a set of general constraints on distributions of costs or rewards.
Thus, a safe policy is defined as a policy that respects some constraints in expectation or in probability.
Our goal is to learn among safe policies one that optimizes the usual expected discounted sum of rewards.
Furthermore, we also require safe learning, i.e., the safety constraints shall be satisfied during training as well.

To that aim, we first propose a general framework that accepts various safety formulations from bounds on CVaR to variance, to probability of reaching bad states.
This general framework is made amenable by formulating the problem in the distributional RL setting, where distributions of returns are learned in contrast to their expectations. 
Based on this general distributional formulation, we extend an existing safe RL algorithm, Interior-Point Policy Optimization (IPO) \citep{liu2019ipo}, to the distributional setting, for which we formulate a performance bound.
\paragraph{Contributions}
Our contributions are threefold: 
(1) We propose a general framework for safe RL where safety is expressed as the satisfaction of risk constraints, which is enforced during and after training.
A risk constraint can be expressed as any (sub)differentiable function of a random variable representing a cumulative reward or cost.
(2) In order to obtain a practical algorithm, we formulate our problem and solution method in the distributional RL setting.
(3) Our proposition, called SDPO, is empirically validated on multiple domains with various risk constraints against relevant state-of-the-art algorithms.

\section{Related Work}\label{sec:related}

Safe RL is becoming an important research direction in RL \citep{GarciaFernandez15}.
In this paper, we distinguish three main non-exclusive aspects for safe RL:
policy safety, 
algorithmic safety, and
exploration safety in exploration.

\begin{table}[t]\small
    \centering
    \begin{tabular}{@{}cccccc@{}}
    \toprule
    Algorithm & PD & CPO &  IPO & PCPO & SDPO \\
    \midrule
    Expectation & \cmark & \cmark & \cmark & \cmark & \cmark\\
    Variance & \cmark & \xmark & \xmark & \xmark & \cmark\\
    CVaR & \cmark & \xmark & \xmark & \xmark & \cmark\\
    (Sub)differentiable fun. & \xmark & \xmark & \xmark & \xmark & \cmark\\
    \midrule
    Safe learning & \xmark & \cmark & \cmark & \cmark & \cmark\\
    Safe execution & \cmark & \cmark & \cmark & \cmark & \cmark\\
    \bottomrule
    \end{tabular}
    \caption{Summary of related algorithms: which constraints are accepted, whether safety is guaranteed during learning/execution. PD (primal-dual) actually corresponds to several algorithms.}
    \label{tab:summary}
\end{table}

Policy safety corresponds to the goal of learning a safe policy such that its execution would avoid/limit the occurrence of bad outcomes (e.g., probability of reaching bad states or bound on performance).
Safety can be modeled as additional constraints or penalization.
In that sense, safe RL is related to risk-sensitive RL 
\citep{Borkar10,chow2014algorithms,ChowTamarMannorPavone15} where the goal is to learn a policy that optimizes a risk-sensitive objective function,  constrained RL \citep{AchiamHeldTamarAbbeel17,tessler2018reward,miryoosefi2019reinforcement,liu2019ipo} where the goal is to learn a policy that optimizes some constraints, and risk-constrained RL \citep{geibel2005risk,BorkarJain14,prashanth2016variance,ChowGhavamzadehJansonPavone16,brazdil2020reinforcement}, which in some sense combines the previous settings.
The works in those three areas, with a few exceptions (CPO \citep{AchiamHeldTamarAbbeel17}, IPO \citep{liu2019ipo}, PCPO \citep{yang2020projection}), do not provide any safety guarantee during learning.
They are based on a primal-dual approach (PD).
For the exceptions, they can only accept simple constraints on expected discounted total costs.
Notably, our algorithm, called Safe Distributional Policy Optimization (SDPO), builds on IPO \citep{liu2019ipo} and extends it to the distributional RL setting, which then allows the formulation of sophisticated constraints.
See Table~\ref{tab:summary} for a summary.

Algorithmic safety corresponds to the idea that running a safe RL algorithm should also guarantee some safety property, e.g., continuous policy improvement \citep{PirottaRestelliPecorinoCalandriello13}, convergence to stationary point \citep{yu2019convergent}, satisfaction of logic specifications \citep{alshiekh2017safe},  satisfaction of  constraints \citep{AchiamHeldTamarAbbeel17,yang2020projection} during learning.
However, none of those propositions can take into account sophisticated safety constraints (e.g., on risk measure).

Exploration safety focuses on an important aspect of safe RL: the exploration problem during learning in order to limit/avoid selecting dangerous actions..
In this context, safety is generally modeled as avoiding bad states.
One main line of work \citep{turchetta2016safe,berkenkamp2017safe,WachiSuiYueOno18,cheng2019end} tries to prevent the choice of a bad action by learning a model.
Other directions have been explored, for instance, by using a verification method \citep{fulton2018safe} or by correcting a chosen action 
\citep{dalal2018safe}.
However, this type of approaches requires the assumption that the environment is deterministic.

Although research has been active in safe RL, to the best of our knowledge, no efficient algorithm has been proposed for the general framework that we propose. 
In particular, our proposition can learn a risk-constrained policy while ensuring the satisfaction of the risk constraint during learning.
Our proposition is based on distributional RL \citep{bellemare2017distributional}, which has demonstrated that estimating distributions of returns instead of their expectations can ensure better overall performance of RL algorithms.
Most work \citep{dabney2018implicit,NeurIPS2019_8850} in this area focuses on value-based methods, extending mostly the DQN algorithm \citep{MnihKavukcuogluSilverRusuVenessBellemareGravesRiedmillerFidjelandOstrovskiPetersenBeattieSadikAntonoglouKingKumaranWierstraLeggHassabis15}.
However, one recent work has also investigated the extension of the distributional setting to policy optimization \citep{barth2018distributed}.
Our work is based on the IQN algorithm \citep{dabney2018implicit} instead of more recent propositions (e.g., \citep{NeurIPS2019_8850}) because of its simplicity and because it perfectly fits our purposes.
Note that in IQN, the authors consider optimizing a risk-sensitive objective function, but they do not consider constraints, as we do.


\section{Background}\label{sec:background}

In this section, we present the notations, recall the definition of a Markov Decision Process (MDP) as well as its extension to Constrained Markov Decision Process (CMDP), and review the notions (e.g., CVaR) and the related deep RL algorithms, which we use to formulate our method.

\paragraph{Notations}
For any set $X$, $\Delta(X)$ denotes the set of probability distributions (or densities if $X$ is continuous) over $X$.
For any function $f: Y \to \Delta(X)$ and
any $(x, y) \in X \times Y$, $f(x \mid y)$ denotes the probability (or density value if $X$ is continuous) of obtaining $x$ according to $f(y)$.
For any $n \in \mathbb N$, $[n]$ denotes $\{1, 2, \ldots, n\}$.
Vectors (resp. matrix) will be denoted in bold lowercase (resp. uppercase) with their components in normal font face with indices.
For instance, $\bm v = (v_1, \ldots, v_n) \in \mathbb{R}^n$ or $\bm M = (m_{ij})_{i \in [n], j \in [m]} \in \mathbb{R}^{n \times m}$.

\paragraph{MDP Model}
A \textit{Markov Decision Process} (MDP) \citep{sutton2018reinforcement} is defined as a tuple $(\St,\Ac,\T,\R,\id,\gamma)$, where $\St$ is a set of states, $\Ac$ is a set of actions, $\T:\St \times \Ac \to \Delta(\St)$ is a transition function, 
$\R: \St \times \Ac \to \mathbb{R}$ is a reward function, 
$\id \in \Delta(\St)$ is a distribution over initial states, and $\gamma \in [0, 1)$ is a discount factor. 
In this model, a policy $\pi: \St \to \Delta(\Ac)$ is defined as a mapping from states to distributions over actions.
We also use notation $\pi_\vpp$ to emphasize that the policy is parameterized by $\vpp$ (e.g., parameters of neural network).
In the remaining, we identify $\pi_\vpp$ to its parameter $\vpp$ for ease of  notation.
The usual goal in an MDP is to search for a policy that maximizes the expected
discounted total reward:
\begin{equation}
\J(\vpp)=\mathbb{E}_{\id, \T, \pi_\vpp}[{\textstyle\sum_{t=0}^{\infty}\gamma^t\R(s_t,a_t )}] 
\end{equation}
where $\mathbb E_{\id, \T, \pi_\vpp}$ is the expectation with respect to the distribution $\id$, the transition function $\T$, and $\pi_{\vpp}$.
We define the \textit{(state) value function} of a policy $\pi_\vpp$ for state $s$ as:
\begin{equation}
\V^\vpp(s)=\mathbb{E}_{\T,\pi_\vpp} [{\textstyle\sum_{t=0}^{\infty}\gamma^t\R(s_t,a_t )}|s_0=s]
\end{equation}
where $\mathbb E_{\T, \pi_\vpp}$ is the expectation with respect to the transition function $\T$ and $\pi_\vpp$.
The \textit{(action) value function} is defined as follows:
\begin{equation}
\Q^\vpp(s,a)=\mathbb{E}_{\T,\pi_\vpp}[{\textstyle\sum_{t=0}^{\infty}\gamma^t\R(s_t,a_t)}|s_0=s,a_0=a]
\end{equation}
and the \textit{advantage function} is defined as:
$\A^\vpp(s,a)=\Q^\vpp(s,a)-\V^\vpp(s)$.
As there is no risk of ambiguity, to avoid clutter we drop $\id$ and $\T$ in the notation of the expectation from now on. 

Reinforcement learning (RL) is based on MDP, but in RL, the transition and reward functions are not assumed to be known.
Thus, in (online) RL, an optimal policy needs to be learned by trial and error.

\paragraph{CMDP Model}
The MDP model can be extended to the \textit{Constrained MDP} (CMDP) setting \citep{Altman99} in order to handle constraints.
In a CMDP, $\nC$ cost functions $\Cost_i:\St\times \Ac\to \mathbb{R}$ for $i \in [\nC]$ are introduced in addition to the original rewards.
For each cost function $\Cost_i$, the corresponding value functions can be defined.
They are denoted with a subscript, e.g., $\J_{\Cost_i}$, $\V_{\Cost_i}$, or $\Q_{\Cost_i}$. 
For a CMDP, the goal is to find a policy that maximizes the expected discounted total reward while satisfying constraints on the expected costs $\J_{\Cost_i}(\vpp)$: 
\begin{equation}
\max_\vpp \J(\vpp) 
\mathrm{~s.t.~} \J_{\Cost_i}(\vpp) 
\le \Cb_i \quad \forall i \in [\nC], \label{eq:constraintCMDP}
\end{equation}
where $\vCb = (\Cb)_{i \in [\nC]} \in \mathbb R^\nC$ is a fixed vector constraint bound.

\paragraph{Proximal Policy Optimization}

The family of policy gradient methods constitutes the standard approach for tackling an RL problem when considering parametrized policies.
Such a method iteratively updates a policy parameter in the direction of a gradient given by
\citep{sutton2018reinforcement}:
\begin{equation*}
 \nabla_\vpp \J(\vpp)=\mathbb{E}_{(s, a) \sim \sd^{\pi_\vpp}}[\A^\vpp(s, a) \nabla_\vpp\log\pi_\vpp(a \mid s) ]
\end{equation*}
where the expectation is taken with the respect to the state-action visitation distribution of $\pi_\vpp$.
One issue in applying a policy gradient method is the difficulty of estimating $\A^\vpp$ online.
This issue motivates the use of an actor-critic scheme where an actor ($\pi_\vpp$) and a critic (e.g., $\A^\vpp$ or $\V^\vpp$ depending on the specific algorithm) are simultaneously learned. 
Learning the value function can help the policy update, such as reducing the gradient variance.

Proximal Policy Optimization (PPO)  \citep{SchulmanWolskiDhariwalRadfordKlimov17} is a state-of-the-art actor-critic algorithm, which optimizes instead a clipped surrogate objective function $\J_{PPO}(\vpp)$ defined by:
\begin{align}
\textstyle\sum_{t=0}^\infty\min(\omega_t(\vpp)\A^{\bar\vpp}(s_t,a_t),\text{clip} (\omega_t(\vpp),\epsilon)\A^{\bar\vpp}(s_t,a_t)),   
\end{align}
where $\bar\vpp$ is the current policy parameter, $\omega_t(\vpp)=\frac{\pi_\vpp(a_t|s_t)}{\pi_{\bar\vpp}(a_t|s_t)}$, and clip$(\cdot,\epsilon)$ is the function to clip between $[1-\epsilon,1+\epsilon]$.
This surrogate function was motivated as an approximation of that used in TRPO \citep{SchulmanLevineAbbeelJordanMoritz15}, which was introduced to ensure monotonic improvement after a policy parameter update.
Although PPO is more heuristic than TRPO, its advantages are its simplicity and lower sample complexity.

\paragraph{Distributional Reinforcement Learning}\label{subsec:DistributionalRL}

The key idea in distributional RL \citep{bellemare2017distributional} is to learn a random variable to represent the discounted return $\Zf^\vpp(s, a)=\sum_{t=0}^{\infty}\gamma^t \vR_t$ where $\vR_t$ is the random variable representing the immediate reward received at time step $t$ when applying action $a$ in state $s$ and policy $\pi_\vpp$ thereafter.
In contrast, standard RL algorithms directly estimate the expectation of $\Zf^\vpp(s, a)$, since 
$\Q^\vpp(s, a)=\mathbb{E}_{\Zf^\vpp}[\Zf^\vpp(s, a)]$ where the expectation is with respect to the distribution of $\Zf^\vpp(s, a)$.


Recall that any real random variable $\Z$ can be represented by its cumulative distribution denoted $F_\Z(z) = \mathbb P(\Z \le z) \in [0, 1]$, or equivalently by its quantile function (inverse cumulative distribution) denoted $F^{-1}_\Z(p) = \inf \{ z\in \mathbb R \mid p \le F_\Z(z)\}$ for any $p \in [0, 1]$.
For ease of notation, $Z_p$ denotes the $p$-\textit{quantile} $F^{-1}_\Z(p)$.
In the \textit{Implicit Quantile Network} (IQN) algorithm, \citet{dabney2018implicit} proposed to approximate the quantile function of $\Zf(s, a)$ with a neural network and to learn it using quantile regression \citep{Koenker05}.

Concretely, the quantile function of $\Zf(s, a)$ can be learned as follows.
Denote $\hat\Zf(s, a)$ the approximated random variable whose quantile function is given by a neural network $\NN(s, \tau)$, which takes as input a state $s$ and a probability $\tau \in [0, 1]$ and returns the corresponding $\tau$-quantile $\hat\Zf_\tau(s, a)$ for each action $a$.
After observing a transition $(s, a, r, s')$, $\NN$ can be trained by sampling $2N$ values $\bm\tau = (\tau_1, \ldots, \tau_N)$ and $\bm\tau' = (\tau'_1, \ldots, \tau'_N)$ with the uniform distribution on $[0, 1]$.
By inverse transform sampling, sampling $\bm\tau$ amount to sampling 
$N$ values from $\hat\Zf(s, a)$ corresponding to $\hat\Zf_{\tau_1}(s, a), \ldots, \hat\Zf_{\tau_N}(s, a))$,
and similarly for $\bm\tau'$ and sampling from $\hat\Zf(s', \pi(s'))$ where $\pi$ is the current policy.
Those samples define $N^2$ TD errors in the distributional setting: 
\begin{equation}
\delta_{ij} = r + \gamma \hat\Zf_{\tau'_j}(s', \pi(s')) - \hat\Zf_{\tau_i}(s, a)
\end{equation}
Following quantile regression, the following loss function for training the neural network $\NN$ in $(s, a, r, s')$ is given by:
\begin{equation}
L_{IQN} = \frac{1}{N}\textstyle\sum_{i \in [N]} \sum_{j \in [N]} \xi_{\tau_i}^\kappa(\delta_{ij}) \label{eq:LIQN}
\end{equation}
where for any $\tau \in (0, 1]$,
$\xi_{\tau}^\kappa(\delta_{ij})=|\tau-\mathbb{I}(\delta_{ij}<0)|\frac{L_\kappa(\delta_{ij})}{\kappa}$
is the quantile Huber loss with threshold $\kappa$ with 
$L_\kappa(\delta)=
\frac{1}{2}\delta^2$ for $|\delta|\le\kappa$ or $\kappa(|\delta|-\frac{1}{2}\kappa)$ otherwise.

\paragraph{Interior-Point Policy Optimization}\label{subsec:IPO}

In the CMDP setting, 
Interior-point Policy Optimization (IPO)  \citep{liu2019ipo} is a recent RL algorithm to maximize an expected discounted total rewards while satisfying constraints on some expected discounted total costs.
To deal with a constraint, IPO augments PPO's objective function with a logarithmic barrier function applied to it, which provides a smooth approximation of the indicator function.
The constrained problem then becomes an unconstrained one with an augmented objective function:
\begin{equation}
\max_\vpp \J_{IPO}(\vpp)=\J_{PPO}(\vpp) + \textstyle\sum_{i\in [\nC]}\frac{\ln(\Cb_i - \J_{\Cost_i}(\vpp))}{\eta}, \label{eq:ipo}
\end{equation}
where $\eta$ is a hyper-parameter.
As $\eta$ tends to $\infty$, the solution of \eqref{eq:ipo} tends to that of the original constrained problem.
The objective $J_{IPO}$ is differentiable, therefore, we can apply a gradient-based optimization method to update the policy.

\section{Problem Formulation}\label{sec:problem}

Let $\Delta(\mathbb R)$ denote the set of real random variables. 
Therefore, $\Zf \in \Delta(\mathbb R)^{\St}$ denotes a function from states to random variables.
Given an (unknown) CMDP, the problem tackled in this paper can be expressed as a constrained optimization problem formulated in the distributional RL setting:
\begin{align}
\max_{\vpp}~~ & 
\mathbb{E}_{s_0\sim\id, \Zf^\vpp}[\Zf^{\vpp}(s_0)] \label{eq:prob1}\\
\text{s.t.}~~ & 
\Cf_{i}(\Yf_i^{\vpp}) \le \Cb_i \quad \forall i \in [\nC] \label{eq:prob2}
\end{align}
where $\Zf^\vpp(s)$ corresponds to the return distribution generated by policy $\pi_\vpp$ from the reward function, 
for all $i\in [\nC]$, 
$\Yf^\vpp_i(s)$ corresponds to the cumulated cost distribution from cost function $\Cost_i$, and
$\Cf_i : \Delta(\mathbb R)^{\St} \to \mathbb R$ is a (sub)differentiable function. 
Note that this formulation is strictly more general than problem \eqref{eq:constraintCMDP} thanks to the \paul{possibly non-linear} functions $\Cf_i$'s.

\begin{table}[t]\small
    \centering
    \begin{tabular}{ll}
        \toprule
        $\Cf_i$ & Definition \\
        \midrule
        Expectation & $\mathbb{E}_{s_0\sim\id, \Yf}[\Yf(s_0)]$ \\
        Prob. of bad states & $\mathbb{E}_{s_0\sim\id, \Yf}[\Yf(s_0)]$ \\
        $\alpha$-CVaR of rewards & $ \mathbb{E}_{s_0\sim\id}[ \frac{1}{\alpha} \int_{0}^\alpha \Y_\zeta(s_0) d\zeta ]$ \\
        Variance & $ \mathbb{E}_{s_0\sim\id}[  \mathbb{E}_{\Yf}[\Yf(s_0)^2] - \mathbb{E}_{\Yf}[\Yf(s_0)]^2]$ \\
        \bottomrule
    \end{tabular}
    \caption{Common examples for $\Cf_i$.}
    \label{tab:rho}
\end{table}

We recall a few common cases for $\Cf_i$ in Table~\ref{tab:rho}.
The expectation is a simple example.
For episodic MDPs with absorbing bad states, another simple example is the probability of bad states, which is defined like the expectation, but applied to a undiscounted cost equal to $1$ for a bad state and $0$ otherwise.
%
CVaR is a widely-used risk measure in finance.
In this context, the $\alpha$-CVaR of a portfolio is intuitively its expected return in the worst $\alpha\times100\%$ cases.
Here, we adapted the definition to rewards (instead of costs).
Naturally, a CVaR of an additional cost would also be possible.
In contrast to previous methods, our framework can accept any (sub)differentiable definitions for $\Cf_i$ (e.g., coherent risk measures).

Note that we chose to take the mean (over initial states) of the CVaRs instead of the CVaR of the mean.
The latter would have been possible as well, but because CVaR is a convex risk measure, our definition is an upperbound of the CVaR of the mean, which means that our formulation is more conservative and in that sense, safer.
The same trick applies if $\Cf_i$ were defined based on any other coherent risk measure, of which CVaR is only one instance.
%
Similarly, for the variance, we use
the mean (over initial states) of variances instead of the other way around.
Since the initial states are sampled in an independent way, the $\Yf(s_0)$'s are independent.
This means that our definition upperbounds  the variance of the mean of the $\Yf(s_0)$'s, leading to a more cautious formulation, which is more desirable for safe RL.


In this paper, we define a \textit{safe policy} as a policy satisfying constraints \eqref{eq:prob2}.
Our goal is to learn a policy maximizing an expected discounted total rewards \eqref{eq:prob1} among all safe policies (i.e., safe execution).
Besides, we require that any policy used during learning be safe (i.e., safe learning).

The formulation of \eqref{eq:prob1}-\eqref{eq:prob2} in the distributional RL setting serves two purposes.
First, as observed in distributional RL, estimating the distributions of the cumulated rewards improves the overall performance.
Second, many safety constraints \eqref{eq:prob2}, such as CVaR, become natural and simple to express in the distributional setting. 

\section{Proposed Method}\label{sec:method}

To solve problem \eqref{eq:prob1}-\eqref{eq:prob2} in the safe RL setting, we extend IPO to the distributional RL setting and combine it with an adaptation of IQN.
Next, we explain the general principle of our approach, and then discuss some techniques to obtain a concrete efficient implementation.

\subsection{General Principle}




To adapt IPO, we rewrite the surrogate objective function used in PPO in the distributional setting:
\begin{align}
\J_{PPO}(\vpp) & = \sum_{t=0}^\infty \min(\omega_t(\vpp) \mathbb E_\vpp[\Zf^{\bar\vpp}(s_t, a_t) - \Zf^{\bar\vpp}(s_t)], \notag\\
&  \text{clip} (\omega_t(\vpp),\epsilon) \mathbb E_\vpp[\Zf^{\bar\vpp}(s_t, a_t) - \Zf^{\bar\vpp}(s_t)]).
\end{align}
Problem \eqref{eq:prob1}-\eqref{eq:prob2} can then be tackled by iteratively solving the following problem with this surrogate function:
\begin{equation}
\max_{\vpp}~~  
\J_{PPO}(\vpp) ~~
\text{s.t.}~~ 
\Cf_{i}(\Yf_i^{\vpp}) \le \Cb_i \quad \forall i \in [\nC]. \label{eq:prob2PPO}
\end{equation}
Now, following IPO, using the log barrier function, we reformulate problem \eqref{eq:prob2PPO} as an unconstrained problem:
\begin{align}
\max_{\vpp}~~ & 
\J_{PPO}(\vpp) + \sum_{i \in [\nC]} \frac{\ln(\Cb_i - \Cf_{i}(\Yf_i^{\vpp}))}{\eta_i}. \label{eq:prob}
\end{align}
In contrast to convex optimization \citep{BoydVandenberghe04}, we introduce a constraint-dependent hyperparameter $\eta_i$ to better control the satisfaction of each constraint \paul{separately}.


Finally, we propose to solve problem~\eqref{eq:prob} with an actor-critic architecture where both the actor and the critic are approximated with neural networks.
For the critic, we adapt the approach proposed for IQN \citep{dabney2018implicit} to learn random returns $\Zf$ and random cumulated costs $\Yf_i$'s.
For the actor, parameter $\vpp$ of policy $\pi_{\vpp}(a|s)$ is updated in the direction of the gradient of the objective function \paul{defined} in \eqref{eq:prob}:
\begin{align}
    \nabla_\vpp \J_{PPO}(\vpp) - \sum_{i \in [\nC]} \frac{1}{\eta_i} \frac{\nabla_\vpp \Cf_i(\Yf^\vpp_i)}{\Cb_i - \Cf_i(\Yf^\vpp_i)}.\label{eq:gradient}
\end{align}
\paul{This gradient raises one difficulty regarding the computation of $\nabla_\vpp \Cf_i(\Yf^\vpp_i)$, which corresponds to the gradient of a critic with respect to the parameters of the actor.
When $\Cf_i$ is linear (i.e., for expectation constraints), the policy gradient theorem \citep{SuttonMcAllesterSinghMansour00} applies and specifies how to compute $\nabla_\vpp \Cf_i(\Yf^\vpp_i)$.
However, when $\Cf_i$ is non-linear (i.e., for more sophisticated risk constraints), the gradient in \eqref{eq:gradient} cannot be obtained easily.
To solve this issue, we propose a simple and generic solution, which consists in connecting the actor network to any critic network with a non-linear $\Cf_i$ (see \Cref{fig:structure} for an illustration where only one critic corresponding to non-linear $\Cf_i$ is displayed).
Using this construct,
}
the exact gradient of $\Cf_i(\Yf^\vpp_i)$ can be computed by automatic differentiation if $\Cf_i$ is (sub)differentiable and $\Yf^\vpp_i$ is approximated with a neural network, as we assume.
\paul{Note that in previous work, \cite{dabney2018implicit} who proposed to optimize a risk measure in IQN did not face this gradient issue because their algorithm is based on DQN \citep{MnihKavukcuogluSilverRusuVenessBellemareGravesRiedmillerFidjelandOstrovskiPetersenBeattieSadikAntonoglouKingKumaranWierstraLeggHassabis15} and therefore does not have an actor network.
As a side note, this construct could be used to deal with a more general problem than \eqref{eq:prob1}-\eqref{eq:prob2} where a non-linear transformation is also applied on the objective function.
For instance, one may want to optimize the CVaR of some rewards subject to some other risk constraints, which is as far as we know a completely novel problem.
We leave this investigation to future work.
}
\begin{figure}[t]
	\centering
	\includegraphics[trim=10pt 0pt 0pt 0pt,clip,width=3in]{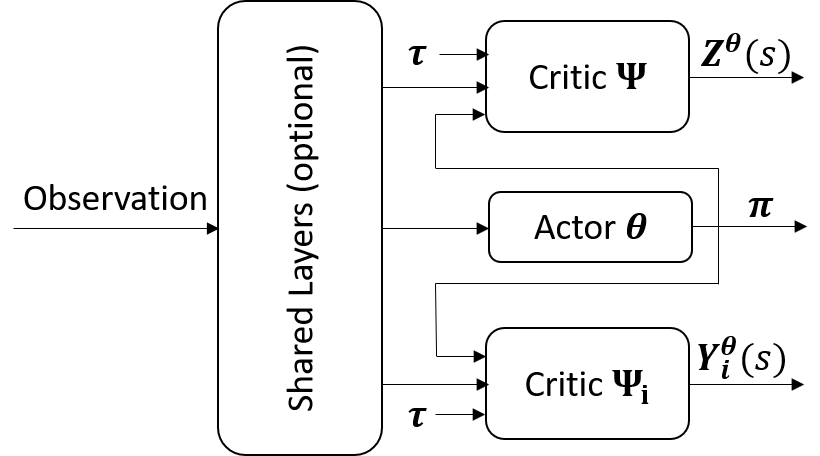}
	\caption{Architecture of SDPO \paul{where critic $\Psi$ corresponds to the objective function and critic $\Psi_i$ corresponds to constraint $i$. Both critics outputs a distribution}.}
	\label{fig:structure}
\end{figure}

Like any interior point method, an initial feasible (i.e., safe) solution is needed.
This requirement is actually not as strong as it seems.
In many MDPs (or CMDPs), there is a known safe action for every state.
For instance, in navigation problem, the action of not moving is safe if the current state is safe.
In finance, investing in cash or a risk-free asset is safe.
For many problems, a dummy action that does not have any effect can be added to define an initial safe action.
More generally, when such a simple safe policy cannot be defined, an expert could possibly provide this initial safe policy or it could be obtained by pretraining with an imperfect simulator.

\subsection{Techniques for Efficient Implementation}

In this section, to simplify notations, we do not write the superscript $\vpp$ for the random variables $\Zf$ and $\Yf_i$'s.

To make our final algorithm more efficient, we propose to learn $\Zf(s)$ only, instead of $\Zf(s, a)$ as it is the usual practice in distributional RL. 
This serves two purposes: (1) a state-dependent distribution is easier to learn, and (2) the advantage function can be easily estimated from a state value function alone.
Note that for the constraints only $\Yf_i(s)$ is needed for any $i \in [\nC]$.
Recall that the two random variables $\Zf(s)$ and $\Zf(s, a)$ are related by the following equation:
\begin{align}
    \Zf(s, a) = R(s, a) + \gamma\mathbb E_{s' \sim \T(\cdot \mid s, a)}[ \Zf(s')]
\end{align}
Following IQN, random variable $\Zf(s)$ is approximated by a random variable $\hat\Zf$, which is represented by a neural network.
The expectation of $\Zf(s)$ can then be approximated by that of $\hat\Zf(s)$ with $\bm\tau$ randomly uniformly sampled in $[0, 1]$:
\begin{equation}
\mathbb E_\vpp[ \Zf(s) ] \approx \sum_{i=1}^{N}(\tau_i-\tau_{i-1})  \hat\Zf_{\tau_i}(s).
\label{eq:Z}
\end{equation}
setting $\tau_0=0$ by convention and assuming $0 < \tau_1 < \tau_2 < \ldots < \tau_N < 1$.

The exact handling of the constraints depend on the definition of $\Cf_i$.
As illustrative examples, we explain how they can be computed for some concrete cases.
If $\Cf_i$ is simply defined as an expectation, it can be dealt with like the objective function.
For CVaR, it can be estimated as follows for a random variable $\Yf(s_0)$:
\begin{equation}
c_\alpha(\Yf) \approx c_\alpha(\hat\Yf) = \frac{1}{\alpha}\sum_{i \mid \tau_i\le\alpha} (\tau_i-\tau_{i-1}) \hat\Yf_{\tau_i}(s_0) \label{eq:c_alpha}
\end{equation}
Here, in contrast to the standard expectation (e.g., \eqref{eq:Z}), an implementation trick consists in sampling $\bm\tau$ in $[0, \alpha]$ such as $\tau_1 < \tau_2 < \ldots < \tau_N=\alpha$ since \eqref{eq:c_alpha} corresponds to the expectation conditioned on event ``$\hat\Yf \le \hat\Yf_{\alpha}$''.
For the variance, $\Cf_i(\Yf)$ can be estimated by:
\begin{equation}
    \sum_{i=1}^{N}(\tau_i-\tau_{i-1}) \hat\Yf_{\tau_i}(s_0)^2 - \left(\sum_{i=1}^{N} (\tau_i-\tau_{i-1}) \hat\Yf_{\tau_i}(s_0) \right)^2
\end{equation}

The pseudo code of our method is shown in Algorithm~\ref{alg:method}.
\begin{algorithm}[t]\small
	\caption{SDPO} \label{alg:method}
	\begin{algorithmic}[1]
	\REQUIRE Constraint bound $\vCb$,
	Initial policy network $\pi_{\vpp_0}$, 
	Initial IQN network $\bm \NN_0$,
	Hyperparameters $\epsilon$ for PPO clip rate and $\eta_i$ for each logarithmic barrier function.
	\FOR{$k=0, 1, \ldots$}
		\STATE $\mathcal B \gets$ run policy $\pi_{\vpp}$ for N trajectories
		\STATE \# update the IQN network
		\STATE Sample $\tau_1<$ ... $<\tau_N$ from $\mathcal U[0,1]$
		\STATE \# quantile regression
		\FOR {$i,j\in[N]$}
			\STATE $\delta_{ij} = r + \gamma \hat\Zf_{\tau'_j}(s', \pi(s')) - \hat\Zf_{\tau_i}(s, a)$
		\ENDFOR
		
		\STATE Update $\bm\NN_{k+1}$ with $\nabla L_{IQN}$ (see \eqref{eq:LIQN}) using $\mathcal B$
		\STATE Update $\vpp_{k+1}$ with $\nabla J(\vpp_k)$ defined in \eqref{eq:gradient} using $\mathcal B$ 
	\ENDFOR
	
	\end{algorithmic}
	
	\label{algorithm1}
\end{algorithm}

\subsection{Performance Guarantee Bound}

For fixed $\bm\eta$, solving \eqref{eq:prob} instead of \eqref{eq:prob2PPO} may incur a performance loss, which can be bounded under natural conditions, which we discuss below.
Since this result uses weak Lagrange duality, we first recall the definition of the Lagrangian of \eqref{eq:prob2PPO}:
\begin{align*}
    \mathcal L(\vpp, \bm\lambda) = \J_{PPO}(\vpp) + \sum_{i \in [\nC]} \lambda_i (\Cb_i - \Cf_i(\Yf^\vpp_i))
\end{align*}
and its dual function:
$ 
    g(\bm\lambda) = \max_{\vpp} \mathcal L(\vpp, \bm\lambda)
$.
The following bound can be proven:
\begin{theorem}
If $\vpp^*_1$ is an optimal solution of \eqref{eq:prob2PPO}, $\vpp^*_2$ is the strictly feasible optimal solution of \eqref{eq:prob} and the unique stationary point of $\mathcal L( \cdot , \bm\lambda^*)$ with $\lambda^*_i = \frac{1}{\eta_i (\Cb_i - \Cf_i(\Yf^{\vpp^*_2}_i))}$
then:
\begin{equation}
\J_{PPO}(\vpp^*_1) - \J_{PPO}(\vpp^*_2) \le \sum_{i \in [d]} \frac{1}{\eta_i}
\end{equation}
\end{theorem}
\begin{proof}
This result generalizes Theorem~1 of \citep{liu2019ipo}, whose proof implicitly uses convexity (which does not hold in deep RL) and follows from the discussion in page 566 of \citep{BoydVandenberghe04}.

We adapt the proof to our more general setting.
We have:
\begin{align}
    \J_{PPO}(\vpp^*_1) &\le g(\bm\lambda^*) \label{eq:step1}\\
    &= \J_{PPO}(\vpp^*_2) + \sum_{i \in [\nC]} \lambda^*_i ( \Cb_i - \Cf_i(\Yf^{\vpp^*_2}_i)) \label{eq:step2}\\
    &= \J_{PPO}(\vpp^*_2) + \sum_{i \in [\nC]} \frac{1}{\eta_i} \label{eq:step3}
\end{align}
Step \eqref{eq:step1} holds by weak duality because $\lambda^*_i \ge 0$ for all $i \in [\nC]$ (since $\vpp^*_2$ is strictly feasible).
Step \eqref{eq:step2} holds because we have by definition of $\vpp^*_2$:
\begin{equation}
    \nabla_\vpp \J_{PPO}(\vpp^*_2) - 
    \sum_{i \in [\nC]} \frac{\nabla_\vpp \Cf_i(\Yf^{\vpp^*_2}_i)}{\eta_i ( \Cb_i - \Cf_i(\Yf^{\vpp^*_2}_i) )} = 0 \label{eq:gradvpp2}
\end{equation}
which implies that $\vpp^*_2$ maximizes $\mathcal L(\cdot, \bm\lambda^*)$ since $\vpp^*_2$ is assumed to be its unique stationary point.
Step~\eqref{eq:step3} holds by definition of $\bm\lambda^*$.
\end{proof}
The conditions in this theorem are natural.
In order to apply an interior point method, the constrained problem needs to be strictly feasible.
The condition on the stationarity of $\vpp^*_2$ is reasonable and can be controlled by setting $\epsilon$ (used in the clipping function of $\J_{PPO}$) small enough.

As a direct corollary, this result implies that if \eqref{eq:prob} could be solved exactly, the error made by algorithm SDPO is controllable via setting appropriate $\eta_i$'s.
Naturally, in the online RL setting, this assumption does not hold perfectly, but this result still provides some theoretical foundation to our proposition.
In the next section, we validate the algorithm using various experimental settings.

\section{Experimental Results}\label{sec:expe}

The experiments are carried out in three different domains to validate our algorithm: random CMDPs, safety gym, as well as financial investment.
See \paul{Appendix~\ref{app:detail}} for details about hyperparameter settings.

Random CMDPs are CMDPs with $N$ states and $M$ actions, where transition probabilities $\T(s' \mid s, a)$ are randomly assigned with $\lceil \ln N\rceil$ positive values for each pair of state-action, and rewards are sampled from a uniform distribution, i.e., $r(s,a)\sim \mathcal U[0,1]$.
In the experiments, we set $N=1000$ and $M=10$.
We consider two cases: a bound over the variance or a bound over the CVaR, both over the distribution of discounted total rewards.

Safety gym \citep{Ray2019} includes a set of environments designed for evaluating safe exploration in RL.
They all correspond to navigation problems where an agent (i.e., Point, Car, Doggo) moves to some random goal positions to perform some tasks (i.e., Goal, Button, Push) while avoiding entering dangerous hazards or bumping into fragile vases.
Each task has two difficulty levels (i.e., 1 or 2). 
See \paul{Appendix~\ref{app:detail}} for more details.
\paul{For space reasons, we only present a selection of results in this domain 
in the main paper.}
More experiment\paul{al results} in \paul{these} Mujoco \paul{environments} are shown in \paul{Appendix~\ref{app:expe}}.

The third domain is the financial stock market.
\paul{The RL agent} can observe the \jianyi{close} prices of the stocks in one day, i.e., the observation $o_t=\bm p_t=(1,p_{1,t},...,p_{N,t})$ for $N$ selected stocks where the first component corresponds to cash.
We further assume that all transactions are dealt at these prices. 
The action of the agent is defined by a portfolio vector, which \paul{corresponds to allocation weights over cash and} stocks, i.e., $a_t=\bm w_{t+1}=(w_{0,t+1},...,w_{N,t+1})$, $w_0$ \paul{(resp. $w_i$ for $i \in [N]$)} is the weight for cash \paul{(resp. stock $i$) and} $\sum_{i=0}^N w_{i,t}=1$.
Naturally, for each stock, we want to maximize the profit. 
Thus, with reward function $\R_t=\ln\sum_{i=0}^N w_{i,t}\frac{p_{i,t}}{p_{i,t-1}}$, optimizing the undiscounted cumulative rewards can maximize the profit.
We set the CVaR boundary $d_1=0$ to avoid any possible loss.
Detailed settings of the experiment are listed in \paul{Appendix~\ref{app:detail}}.

\begin{figure}[t]
\centering
\subfigure[rewards]{
\includegraphics[trim=10pt 10pt 10pt 10pt,clip,width=1.6in]{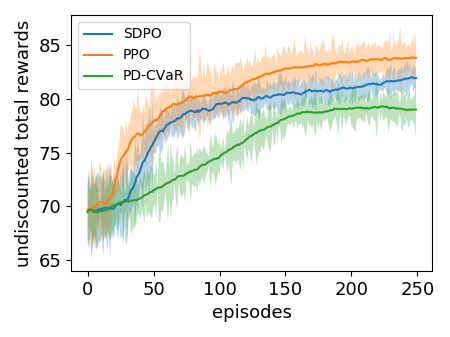}
\label{fig:cmdpcvar1}
}%
\subfigure[constraint]{
\includegraphics[trim=10pt 10pt 10pt 10pt,clip,width=1.6in]{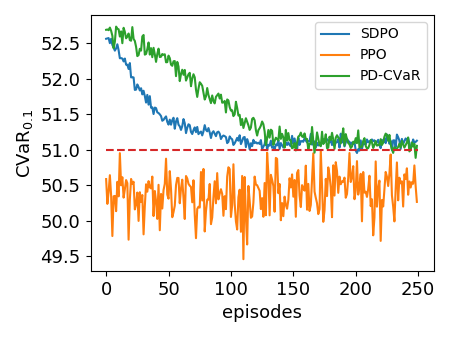}
\label{fig:cmdpcvar2}
}

\caption{\ref{fig:cmdpcvar1}: Average performance over 10 runs of PPO, SDPO and PD-CVaR under the random CMDP for $N=1000$. \ref{fig:cmdpcvar2}: $0.1$-CVaR bounded by $51$. Both SDPO and PD-CVaR converge to the level indicated by the dashed line.
}
\label{fig:cmdpcvar}
\end{figure}

In all our experiments, all the agents are initialized so that they are in a feasible region at the beginning.
In practice, an initial safe policy can be defined using domain knowledge or by an expert, e.g., in Mujoco domain, the agent can be initialized to stay and doing nothing.
For fairness, the PPO agent is also initialized with the same safe policy as all other agents.
Two policy gradient algorithms with CVaR and variance constraints respectively, PD-CVaR \citep{chow2014algorithms} and PD-VAR, which is modified from Algorithm 2 in \citep{prashanth2016variance} \paul{are used} as baselines \paul{in the first domain}.
\paul{SDPO is compared with CPO \citep{AchiamHeldTamarAbbeel17}, PCPO \citep{yang2020projection}, and IPO \citep{liu2019ipo} in the second domain.}
\paul{PPO \citep{SchulmanWolskiDhariwalRadfordKlimov17} is evaluated on all domains to serve as a non-safe RL method.}
\paul{Note that in contrast to our architecture SDPO, none of those algorithms can tackle the problem defined in \eqref{eq:prob1}-\eqref{eq:prob2} in its most general form.}

The experiments are designed \paul{to evaluate SDPO in a variety of domains with various risk constraints and} to answer the following questions:
\textbf{(A)} How does SDPO compare with methods based on Lagrangian relaxation?
\textbf{(B)} How does SDPO compare with other safe RL algorithms? Does the distributional formulation of SDPO help compared to IPO?
\textbf{(C)} How does SDPO perform with multiple constraints (cumulative cost and probability of reaching a bad states)?
\textbf{(D)} How does SDPO \paul{perform} on a real domain? How does the constraint stringency impact the performance of SDPO?


\begin{figure}[t]
\subfigure[rewards]{
\includegraphics[trim=10pt 10pt 10pt 10pt,clip,width=1.6in]{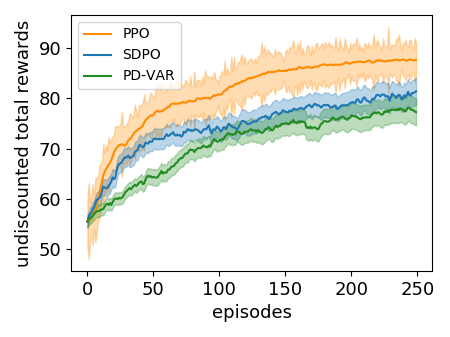}
\label{fig:cmdpvariance1}
}%
\subfigure[constraint]{
\includegraphics[trim=10pt 10pt 10pt 10pt,clip,width=1.6in]{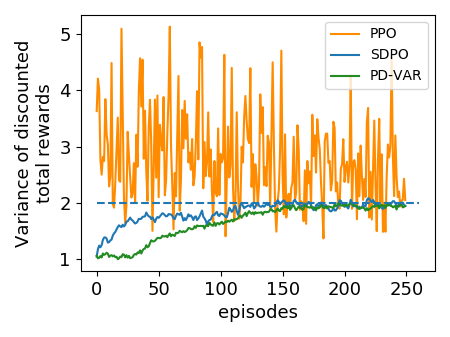}
\label{fig:cmdpvariance2}
}%
\caption{\ref{fig:cmdpvariance1}: Average performance over 10 runs of PPO, SDPO and PD-VAR under the random CMDP for $N=1000$. \ref{fig:cmdpvariance2}: Variance bounded by $2$. Both SDPO and PD-VAR converge to the level indicated by the dashed line.}
\label{fig:cmdpvariance}
\end{figure}

\paragraph{Question (A)} 
To answer (A), we perform some experiments on the first domain, random CMDPs, \paul{with either a constraint on CVaR or a constraint on variance.
Both are based on the rewards. 
Therefore, the first needs to be lower-bounded, while the second needs to be upper-bounded.}
The confidence level is fixed to $\alpha=0.1$ and 
\paul{the bound for CVaR is set to $51$ and that for the variance is set to 2.
The bounds were chosen so that they are not too restrictive.}

From \paul{the} results in \Cref{fig:cmdpcvar}, as expected, PPO without constraint achieves the best total rewards and converges faster than the constrained ones.
When \paul{the} CVaR value is bounded, \paul{PD-CVaR and SDPO both} converge to a slightly worse but safe policy, \paul{however} SDPO converges faster.
From \paul{the} results in \Cref{fig:cmdpvariance}, \paul{similar observations can be drawn for PPO, PD-VAR, and SDPO.
With regards to safety, we can again conclude than SDPO is superior.}


\paragraph{Question (B)} 
To answer (B), we perform some experiments on the second domain, Safety gym, \paul{which is a much more difficult domain than random CMDPs.}
For this domain, we did not evaluate the methods based on Lagrangian relaxation: 
since they do not use a critic, they would not be competitive.
\paul{In Safety gym, the} agent is penalized by receiving a cost $\Cost_1=1$ when touching a fragile vase.
\paul{With a constraint on expected total} cost $\Cf_1(\Yf_1) = \mathbb{E}_{s_0\sim\id, \Yf_1}[\Yf_1(s_0)]\le \Cb_1$, we are able to compare SDPO with other safe RL algorithm\paul{s} like CPO, PCPO and IPO.

\begin{figure}[t]
\centering
\subfigure[rewards]{
\centering
\includegraphics[trim=10pt 10pt 10pt 10pt,clip,width=1.6in]{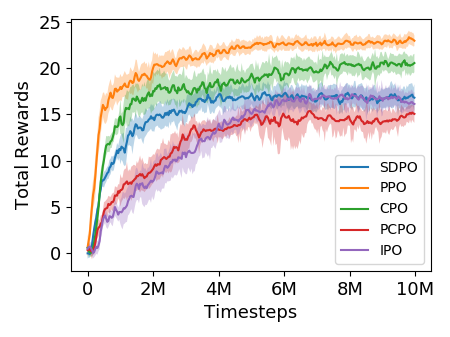}
\label{fig:QBr}
}%
\subfigure[constraint]{
\includegraphics[trim=10pt 10pt 10pt 10pt,clip,width=1.6in]{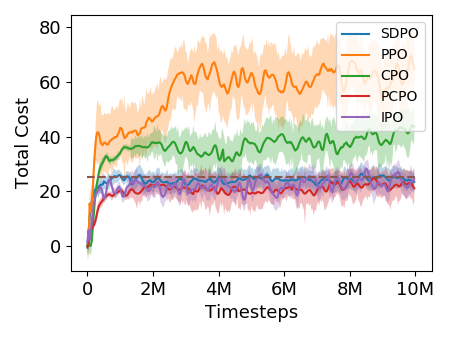}
\label{fig:QBc}
}%
\centering
\caption{Average performance over 10 runs of PPO, SDPO, CPO, PCPO and IPO under Point-Goal1. They are bounded by the dashed line $\Cb_1=25$ in \ref{fig:QBc}.}
\label{fig:QB}
\end{figure}

We only show the \paul{results} for Point-Goal1 in \Cref{fig:QB}. 
For other tasks, please refer to Appendix \ref{app:expe}. 
According to \Cref{fig:QB}, SDPO, PCPO and IPO can explore safely, while CPO cannot satisfy the constraint well.
\paul{This latter observation regarding CPO may be surprising since CPO was designed to solve CMDPs, but similar results were also reported in previous work \citep{Ray2019}.}
Among these three latter algorithms, SDPO and IPO performs the best.
\paul{In \Cref{fig:QB} and in all the Safety-gym environments (see Appendix~\ref{app:expe}), SDPO dominates IPO in terms of either returns or convergence rates (and sometimes both), which confirms the positive contribution of the distributional critics.}

\paragraph{Question (C)}
To demonstrate that SDPO can satisfy multiple constraints, the safety gym environment is used \paul{again, but with a variation}.
We modify the hazard area to be end states where an agent receives a cost $\Cost_2=1$, and the episode is terminated.
\paul{In addition to the previous constraint, another one is enforced:}
$\Cf_2(\Yf_2) = \mathbb{E}_{s_0\sim\id, \Yf_2}[\Yf_2(s_0)]\le \Cb_2$,
where $\Yf_2$ is the undiscounted cumulative cost distribution from cost function $\Cost_2$.
\paul{Here, w}e set the bounds: $\Cb_1=10$ and $\Cb_2=0.1$.

From Figures \ref{fig:mujoco1} and \ref{fig:mujoco2}, PPO without constraints achieves much more goals\paul{, but at the cost of violating all the constraints.}
For constraint $\Cf_2$, both SDPO and IPO agents can avoid entering into hazards during training.
\paul{For} constraint $\Cf_1$, SDPO converges faster than IPO because of the adaption \paul{to} distributional RL.

\begin{figure}[t]
\centering
\subfigure[rewards]{
\centering
\includegraphics[trim=10pt 10pt 10pt 10pt,clip,width=1.6in]{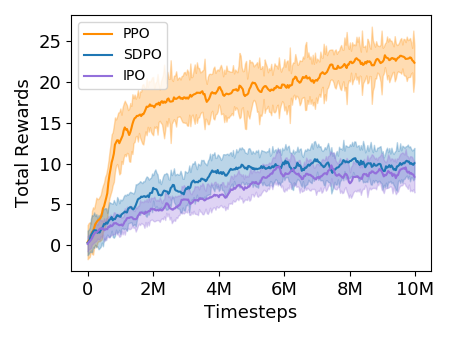}
\label{fig:mujoco1}
}%
\subfigure[constraint\paul{s}]{
\includegraphics[trim=10pt 10pt 10pt 10pt,clip,width=1.6in]{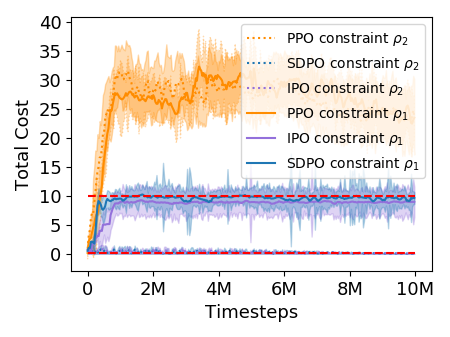}
\label{fig:mujoco2}
}%
\centering
\caption{\ref{fig:mujoco1}: Average performance over 5 runs of PPO, SDPO and IPO under Point-Goal2. \ref{fig:mujoco2}: Average costs of PPO, SDPO and IPO under Point-Goal2.}
\label{fig:mujoco}
\end{figure}

\paragraph{Question (D)}
To answer (D), we switch to the finance domain, where the stock market data \paul{of year 2019} is used.
\paul{
We run SDPO with a constraint on CVaR defined over rewards using different confidence levels $\alpha$.
Note that since the CVaR is defined over rewards, it needs to be lower-bounded.
We also run PPO as a baseline to show the performance without any constraints.
}
\begin{figure}[t]

	\centering
	\includegraphics[trim=10pt 10pt 10pt 10pt,clip,width=1.8in]{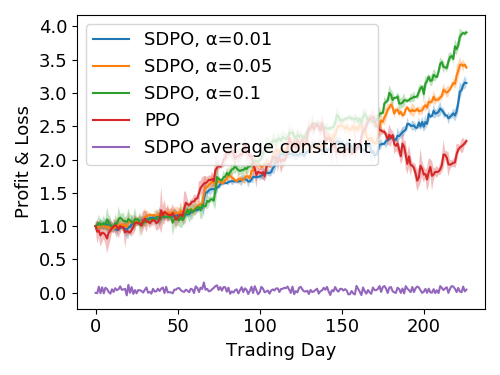}
	\caption{Average performance over 10 runs of PPO and SDPO with confidence level $\alpha=0.01,0.05,0.1$.}
	\label{fig:stock}
\end{figure}

From Figure \ref{fig:stock}, all agents manage to make profits.
With tighter constraint on risk (smaller $\alpha$), the SDPO agent makes less profit.
\paul{While PPO does not satisfy the constraint as expected, the curves for the constraint satisfaction of all SDPO agents are all similar.
We therefore plot their average directly in Figure \ref{fig:stock}.}
PPO without constraint cannot avoid risk and thus suffers from fluctuation and loss at some time point.
\paul{Interestingly, all the SDPO agents eventually perform better than PPO, which demonstrates that enforcing safety does not necessarily prevent good performance.
Finally, SDPO with $\alpha=0.1$ performs best.
}

\section{Conclusion}\label{sec:conclusion}

We presented a general framework for safe RL that encompasses many previous propositions. 
The novelty of our approach is the exploitation of a distributional RL formulation that allows us to deal with sophisticated risk constraints in a natural and efficient way for policy optimization.
Our algorithm, SDPO, is shown to perform well in diverse environments and is competitive with previous algorithms \paul{in situations when they can be applied. 
However}, SDPO can cover a larger range of safety formulations.

\newpage
\bibliography{main}

\begin{thebibliography}{38}
\providecommand{\natexlab}[1]{#1}
\providecommand{\url}[1]{\texttt{#1}}
\expandafter\ifx\csname urlstyle\endcsname\relax
  \providecommand{\doi}[1]{doi: #1}\else
  \providecommand{\doi}{doi: \begingroup \urlstyle{rm}\Url}\fi

\bibitem[Achiam et~al.(2017)Achiam, Held, and et~al.]{AchiamHeldTamarAbbeel17}
J.~Achiam, D.~Held, and A.~Tamar et~al.
\newblock Constrained policy optimization.
\newblock In \emph{ICML}, 2017.

\bibitem[Alshiekh et~al.(2018)Alshiekh, Bloem, R.~Ehlers, K{\"o}nighofer,
  Niekum, and Topcu]{alshiekh2017safe}
M.~Alshiekh, R.~Bloem, R.~R.~Ehlers, B.~K{\"o}nighofer, S.~Niekum, and
  U.~Topcu.
\newblock Safe reinforcement learning via shielding.
\newblock In \emph{AAAI}, 2018.

\bibitem[Altman(1999)]{Altman99}
E.~Altman.
\newblock \emph{Constrained {M}arkov Decision Processes}.
\newblock CRC Press, 1999.

\bibitem[Barth-Maron et~al.(2018)Barth-Maron, Hoffman, and
  Budden]{barth2018distributed}
G.~Barth-Maron, M.~W Hoffman, and D.~et~al. Budden.
\newblock Distributed distributional deterministic policy gradients.
\newblock \emph{ICLR}, 2018.

\bibitem[Bellemare et~al.(2017)Bellemare, Dabney, and
  Munos]{bellemare2017distributional}
M.~Bellemare, W.~Dabney, and R.~Munos.
\newblock A distributional perspective on reinforcement learning.
\newblock \emph{ICML}, 2017.

\bibitem[Berkenkamp et~al.(2017)Berkenkamp, Turchetta, Schoellig, and
  Krause]{berkenkamp2017safe}
F.~Berkenkamp, M.~Turchetta, A.~Schoellig, and A.~Krause.
\newblock Safe model-based reinforcement learning with stability guarantees.
\newblock In \emph{NeurIPS}, 2017.

\bibitem[Borkar and Jain(2014)]{BorkarJain14}
V.~Borkar and R.~Jain.
\newblock Risk-constrained {M}arkov decision processes.
\newblock \emph{IEEE Transactions on Automatic Control}, 59\penalty0
  (9):\penalty0 2574--2579, 2014.

\bibitem[Borkar(2010)]{Borkar10}
V.~S. Borkar.
\newblock Learning algorithms for risk-sensitive control.
\newblock In \emph{International Symposium on Mathematical Theory of Networks
  and Systems}, 2010.

\bibitem[Boyd and Vandenberghe(2004)]{BoydVandenberghe04}
S.~Boyd and L.~Vandenberghe.
\newblock \emph{Convex Optimization}.
\newblock Cambridge university press, 2004.

\bibitem[Brazdil et~al.(2020)Brazdil, Chatterjee, Novotny, and
  Vahala]{brazdil2020reinforcement}
T.~Brazdil, K.~Chatterjee, P.~Novotny, and J.~Vahala.
\newblock Reinforcement learning of risk-constrained policies in {M}arkov
  decision processes.
\newblock \emph{AAAI}, 2020.

\bibitem[Cheng et~al.(2019)Cheng, Orosz, Murray, and Burdick]{cheng2019end}
R.~Cheng, G.~Orosz, R.~M Murray, and J.~W Burdick.
\newblock End-to-end safe reinforcement learning through barrier functions for
  safety-critical continuous control tasks.
\newblock In \emph{AAAI}, 2019.

\bibitem[Chow and Ghavam\-zadeh(2014)]{chow2014algorithms}
Y.~Chow and M.~Ghavam\-zadeh.
\newblock Algorithms for {CVaR} optimization in {MDP}s.
\newblock In \emph{NeurIPS}, 2014.

\bibitem[Chow et~al.(2015)Chow, Tamar, Mannor, and
  Pavone]{ChowTamarMannorPavone15}
Y.~Chow, A.~Tamar, S.~Mannor, and M.~Pavone.
\newblock {Risk-Sensitive and Robust Decision-Making: a {CVaR} Optimization
  Approach}.
\newblock In \emph{NeurIPS}, 2015.

\bibitem[Chow et~al.(2017)Chow, Ghavamzadeh, Janson, and
  Pavone]{ChowGhavamzadehJansonPavone16}
Y.~Chow, M.~Ghavamzadeh, L.~Janson, and M.~Pavone.
\newblock Risk-constrained reinforcement learning with percentile risk
  criteria.
\newblock \emph{JMLR}, 18\penalty0 (1), 2017.

\bibitem[Dabney et~al.(2018)Dabney, Ostrovski, Silver, and
  Munos]{dabney2018implicit}
W.~Dabney, G.~Ostrovski, D.~Silver, and R.~Munos.
\newblock Implicit quantile networks for distributional reinforcement learning.
\newblock \emph{ICML}, 2018.

\bibitem[Dalal et~al.(2018)Dalal, Dvijotham, Vecerik, Hester, Paduraru, and
  Tassa]{dalal2018safe}
G.~Dalal, K.~Dvijotham, M.~Vecerik, T.~Hester, C.~Paduraru, and Y.~Tassa.
\newblock Safe exploration in continuous action spaces.
\newblock \emph{CoRR}, 2018.

\bibitem[Fulton and Platzer(2018)]{fulton2018safe}
N.~Fulton and A.~Platzer.
\newblock Safe reinforcement learning via formal methods.
\newblock In \emph{AAAI}, 2018.

\bibitem[Garcia and Fernandez(2015)]{GarciaFernandez15}
J.~Garcia and F.~Fernandez.
\newblock A comprehensive survey on safe reinforcement learning.
\newblock \emph{JMLR}, 16\penalty0 (1437--1480), 2015.

\bibitem[Geibel and Wysotzky(2005)]{geibel2005risk}
P.~Geibel and F.~Wysotzky.
\newblock Risk-sensitive reinforcement learning applied to control under
  constraints.
\newblock \emph{JAIR}, 24:\penalty0 81--108, 2005.

\bibitem[Jin(2017)]{jinDeepLearningAlibaba2017}
R.~Jin.
\newblock Deep learning at {{Alibaba}}.
\newblock In \emph{IJCAI}, 2017.
\newblock ISBN 978-0-9992411-0-3.
\newblock \doi{10.24963/ijcai.2017/2}.

\bibitem[Koenker(2005)]{Koenker05}
R.~Koenker.
\newblock \emph{Quantile Regression}.
\newblock Cambridge university press, 2005.

\bibitem[Liu et~al.(2020)Liu, Ding, and Liu]{liu2019ipo}
Y.~Liu, J.~Ding, and X.~Liu.
\newblock {IPO}: Interior-point policy optimization under constraints.
\newblock \emph{AAAI}, 2020.

\bibitem[Miryoosefi et~al.(2019)Miryoosefi, Brantley, Daume~III, Dudik, and
  Schapire]{miryoosefi2019reinforcement}
S.~Miryoosefi, K.~Brantley, H.~Daume~III, M.~Dudik, and R.~Schapire.
\newblock Reinforcement learning with convex constraints.
\newblock In \emph{NeurIPS}, 2019.

\bibitem[Mnih et~al.(2015)Mnih, Kavukcuoglu, and
  et~al.]{MnihKavukcuogluSilverRusuVenessBellemareGravesRiedmillerFidjelandOstrovskiPetersenBeattieSadikAntonoglouKingKumaranWierstraLeggHassabis15}
V.~Mnih, K.~Kavukcuoglu, and D.~Silver et~al.
\newblock Human-level control through deep reinforcement learning.
\newblock \emph{Nature}, 2015.

\bibitem[Pirotta et~al.(2013)Pirotta, Restelli, Pecorino, and
  Calandriello]{PirottaRestelliPecorinoCalandriello13}
M.~Pirotta, M.~Restelli, A.~Pecorino, and D.~Calandriello.
\newblock Safe policy iteration.
\newblock In \emph{ICML}, 2013.

\bibitem[Prashanth and Ghavamzadeh(2016)]{prashanth2016variance}
L.~Prashanth and M.~Ghavamzadeh.
\newblock Variance-constrained actor-critic algorithms for discounted and
  average reward {MDP}s.
\newblock \emph{Machine Learning}, 2016.

\bibitem[Ray et~al.(2019)Ray, Achiam, and Amodei]{Ray2019}
A.~Ray, J.~Achiam, and D.~Amodei.
\newblock {Benchmarking Safe Exploration in Deep Reinforcement Learning}.
\newblock 2019.

\bibitem[Schulman et~al.(2015)Schulman, Levine, Abbeel, Jordan, and
  Moritz]{SchulmanLevineAbbeelJordanMoritz15}
J.~Schulman, S.~Levine, P.~Abbeel, M.I. Jordan, and P.~Moritz.
\newblock Trust region policy optimization.
\newblock In \emph{ICML}, 2015.

\bibitem[Schulman et~al.(2017)Schulman, Wolski, Dhariwal, Radford, and
  Klimov]{SchulmanWolskiDhariwalRadfordKlimov17}
J.~Schulman, F.~Wolski, P.~Dhariwal, A.~Radford, and O.~Klimov.
\newblock Proximal policy optimization algorithms.
\newblock \emph{CoRR}, 2017.
\newblock URL \url{http://arxiv.org/abs/1707.06347}.

\bibitem[Silver et~al.(2017)Silver, Schrittwieser, and
  et~al.]{SilverSchrittwieserSimonyanAntonoglouHuangGuezHubertBakerLaiBoltonChenLillicrapHuiSifreDriesscheGraepelHassabis17}
D.~Silver, J.~Schrittwieser, and K.~Simonyan et~al.
\newblock Mastering the game of go without human knowledge.
\newblock \emph{Nature}, 2017.

\bibitem[Sutton and Barto(2018)]{sutton2018reinforcement}
R.~Sutton and A.~G. Barto.
\newblock \emph{Reinforcement learning: An introduction}.
\newblock MIT press, 2018.

\bibitem[Sutton et~al.(2000)Sutton, McAllester, Singh, and
  Mansour]{SuttonMcAllesterSinghMansour00}
Richard~S. Sutton, David McAllester, Satinder Singh, and Yishay Mansour.
\newblock Policy gradient methods for reinforcement learning with function
  approximation.
\newblock In \emph{NeurIPS}, 2000.

\bibitem[Tessler et~al.(2019)Tessler, Mankowitz, and Mannor]{tessler2018reward}
C.~Tessler, D.~Mankowitz, and S.~Mannor.
\newblock Reward constrained policy optimization.
\newblock \emph{ICLR}, 2019.

\bibitem[Turchetta et~al.(2016)Turchetta, Berkenkamp, and
  Krause]{turchetta2016safe}
M.~Turchetta, F.~Berkenkamp, and A.~Krause.
\newblock Safe exploration in finite {M}arkov decision processes with gaussian
  processes.
\newblock In \emph{NeurIPS}, 2016.

\bibitem[Wachi et~al.(2018)Wachi, Sui, Yue, and Ono]{WachiSuiYueOno18}
A.~Wachi, Y.~Sui, Y.~Yue, and M.~Ono.
\newblock Safe exploration and optimization of constrained {MDP}s using
  gaussian processes.
\newblock In \emph{AAAI}, 2018.

\bibitem[Yang et~al.(2019)Yang, Zhao, Lin, Qin, Bian, and
  Liu]{NeurIPS2019_8850}
D.~Yang, L.~Zhao, Z.~Lin, T.~Qin, J.~Bian, and T.~Liu.
\newblock Fully parameterized quantile function for distributional
  reinforcement learning.
\newblock In \emph{NeurIPS}. 2019.

\bibitem[Yang et~al.(2020)Yang, Rosca, Narasimhan, and
  Ramadge]{yang2020projection}
T.~Yang, J.~Rosca, K.~Narasimhan, and P.~Ramadge.
\newblock Projection-based constrained policy optimization.
\newblock In \emph{ICLR}, 2020.

\bibitem[Yu et~al.(2019)Yu, Yang, Kolar, and Wang]{yu2019convergent}
M.~Yu, Z.~Yang, M.~Kolar, and Z.~Wang.
\newblock Convergent policy optimization for safe reinforcement learning.
\newblock In \emph{NeurIPS}, 2019.

\end{thebibliography}

\newpage
\onecolumn
\appendix

\section{More details on Experiments} \label{app:detail}
\subsection{Random CMDP}
In the random CMDP domain, we constructed it with $N=1000$ states and 10 actions.
The number of randomly-chosen possible successor states is $\lceil\ln N\rceil=7$.
An episode in this CMDP is terminated after 100 time-steps.
To achieve the results in Figures \ref{fig:cmdpcvar} and \ref{fig:cmdpvariance}, we trained the agents 5 times for both random seeds 5 and 10.
The ADAM optimizer is used.
The hyper-parameters are listed in Table \ref{tab:cmdp}.
\begin{table}[h]
    \centering
    \begin{tabular}{c|cccc}
         & PPO & SDPO & PD-CVaR & PD-VAR\\
         \midrule
      discount factor $\gamma$   & 0.99 & 0.99 & 0.99 & 0.99\\
      batch size & 1000 & 1000 & 1000 & 1000\\
      learning rate (actor) & 1e-4 & 1e-4 & 1e-4 & 1e-4\\
      learning rate (critic) & 1e-3 & 1e-3 & / & /\\
      hidden sizes & (64,64) & (64,64) & (64,64) & (64,64)\\
      GAE factor $\lambda$   & 0.9 & 0.9 & / & / \\
      clip range $\epsilon$ & 0.2 & 0.2 & / & /\\
      $\eta$ & / & 20 & / & / \\
      \# quantile atoms $N$ & / & 128 & / & /\\
      \ quantile dimension\footnotemark[1] & / & 256 & / & /\\
      \bottomrule
    \end{tabular}
    \caption{Hyperparameters for experiments on random CMDP.}
    \label{tab:cmdp}
\end{table}

To evaluate the converged policies in \ref{fig:cmdpcvar1} with CVaR constraint, we run them for 1000 episodes each. 
\begin{figure}[t]
 	\centering
 	\includegraphics[width=1.85in]{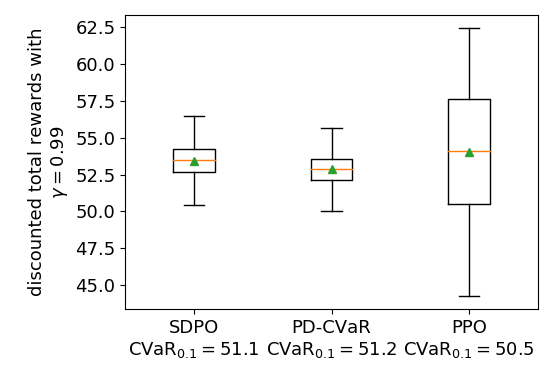}
 	\caption{Evaluation of $\Zf^{\pi_\vpp}$ of trained policy over 1000 runs.}
 	\label{fig:cvarboxpot}
 \end{figure}
Figure \ref{fig:cvarboxpot} indicates that PPO without constraint can reach \paul{the} highest result but suffers \paul{from the} risk of getting the lowest reward.
Both SDPO and PD-CVaR receives a lower mean reward but much higher CVaR values, which indicates lower risk.

\subsection{Stock Transaction}

For the experiment on stock market, we use Quandl\footnote{\url{https://quandl.com}} in Python to load all market data.
The trading agent is assumed to have zero market impact and zero transaction cost.
When conducting this experiment, We choose 9 stocks in SP500 (AAPL, CSCO, DOW, GE, GS, JNJ, JPM, MMM, MSFT). 
The agents are initialized with a safe policy that always holding cash, and then trained in a rolling bias in year 2019 to evaluate the offline performance, i.e., at time step $t$, prices from $t-15$ to $t$ are used for training.
The shared part of the actor and critic network is implemented as an LSTM network.
The hyper-parameters are listed in Table \ref{tab:stock}.
\begin{table}[h]
    \centering
    \begin{tabular}{c|cc}
         & PPO & SDPO\\
         \midrule
      discount factor $\gamma$   & 0.99 & 0.99\\
      batch size & 1280 & 1280\\
      learning rate (actor) & 1e-4 & 1e-4\\
      learning rate (critic) & 1e-3 & 1e-3\\
      hidden sizes & (64,64) & (64,64) \\
      GAE factor $\lambda$   & 0.9 & 0.9\\
      clip range $\epsilon$ & 0.1 & 0.1\\
      $\eta$ & / & 60 \\
      \# quantile atoms $N$ & / & 128\\
      \ quantile dimension & / & 256\\
      \bottomrule
    \end{tabular}
    \caption{Hyperparameters for experiments on stock transaction.}
    \label{tab:stock}
\end{table}
\subsection{Mujoco Simulator}

For the parameters and other settings of the Point-Goal2 domain  
we used the default values set in the source code of Safety-Gym (see line 108 in \href{https://github.com/openai/safety-gym}{\color{blue}safety-gym/safety\_gym/envs/suite.py}).
The hyper-parameters are listed in Table \ref{tab:mujoco}.
\begin{table}[h]
    \centering
    \begin{tabular}{c|ccc}
         & PPO & SDPO & IPO\\
         \midrule
      discount factor $\gamma$ & 0.99 & 0.99 & 0.99 \\
       discount factor for constraints $\gamma_1$, $\gamma_2$ & 1 & 1 & 1 \\
      batch size & 30000 & 30000 & 30000\\
      learning rate (actor) & 1e-4 & 1e-4 & 1e-4\\
      learning rate (critic) & 1e-3 & 1e-3 & 1e-3\\
      hidden sizes (actor) & (256,256) & (256,256) & (256,256) \\
      hidden sizes (critic) & (256,256) & (256,256) & (256,256) \\
      GAE factor $\lambda$   & 0.9 & 0.9 & 0.9\\
      clip range $\epsilon$ & 0.1 & 0.1 & 0.1\\
      $\eta_1$ & / & 40 & 60 \\
      $\eta_2$ & / & 60 & 60 \\
      \# quantile atoms $N$ & / & 128 & /\\
      \ quantile dimension & / & 256 & /\\
      \bottomrule
    \end{tabular}
    \caption{Hyperparameters for experiments on Point-Goal2.}
    \label{tab:mujoco}
\end{table}
\footnotetext[1]{refer to Equation (4) in \citep{dabney2018implicit}}

\section{Additional Experiments} \label{app:expe}

We further compare our method to Constrained Policy Optimization (CPO) \citep{AchiamHeldTamarAbbeel17} and Projection-based Constrained Policy Optimization (PCPO) \citep{yang2020projection}.
PCPO is a two-step approach.
In the first step, the policy is updated in the direction to improve the objective function in the trust region.
In the second step, PCPO projects the potentially infeasible policy back to the constraint set.

To demonstrate the performance of SDPO, compared with PCPO, IPO and CPO, we conduct the experiment in the safety gym environment with constraint $\rho_1$.
Three tasks, Goal, Button and Push with two levels of difficulties are tested in the experiments with agent point, car and dpggo.
Task Goal is to move the agent to a series of goal positions, while task button is to press a series of goal buttons, and task push is to move a box to a series of goal positions.
For detailed explanation of these tasks, please refer to \citep{Ray2019}.
The hyper-parameters in the experiment are the same as \Cref{tab:mujoco}, except $\eta=30$ for SDPO and IPO.
\begin{figure}[h]
\centering
\subfigure[rewards, PointGoal-1]{
\begin{minipage}{0.24\columnwidth}
\centering
\includegraphics[trim=10pt 10pt 10pt 10pt,clip,width=1.6in]{pg1r.png}
\label{fig:pg1r}
\end{minipage}%
}%
\subfigure[constraint, PointGoal-1]{
\begin{minipage}{0.24\columnwidth}
\centering
\includegraphics[trim=10pt 10pt 10pt 10pt,clip,width=1.6in]{pg1c.png}
\label{fig:pg1c}
\end{minipage}%
}
\subfigure[rewards, PointGoal-2]{
\begin{minipage}{0.24\columnwidth}
\centering
\includegraphics[trim=10pt 10pt 10pt 10pt,clip,width=1.6in]{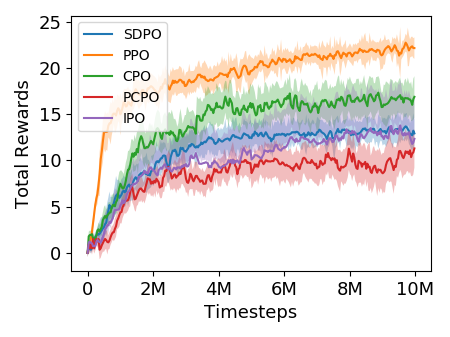}
\label{fig:pg2r}
\end{minipage}%
}%
\subfigure[constraint, PointGoal-2]{
\begin{minipage}{0.24\columnwidth}
\centering
\includegraphics[trim=10pt 10pt 10pt 10pt,clip,width=1.6in]{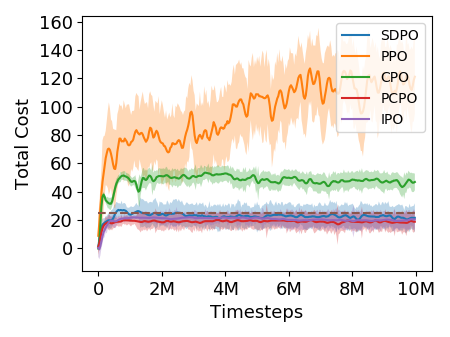}
\label{fig:pg2c}
\end{minipage}%
}
\subfigure[rewards, PointButton-1]{
\begin{minipage}{0.24\columnwidth}
\centering
\includegraphics[trim=10pt 10pt 10pt 10pt,clip,width=1.6in]{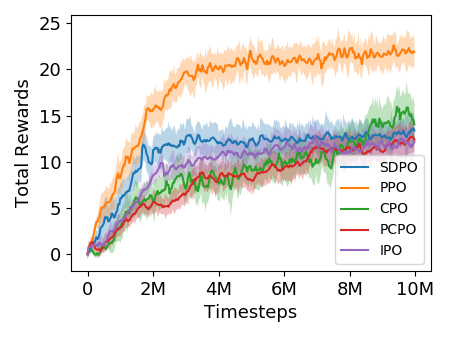}
\label{fig:pb1r}
\end{minipage}%
}%
\subfigure[constraint, PointButton-1]{
\begin{minipage}{0.24\columnwidth}
\centering
\includegraphics[trim=10pt 10pt 10pt 10pt,clip,width=1.6in]{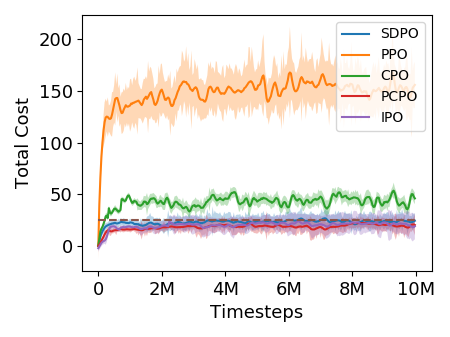}
\label{fig:pb1c}
\end{minipage}%
}
\subfigure[rewards, PointButton-2]{
\begin{minipage}{0.24\columnwidth}
\centering
\includegraphics[trim=10pt 10pt 10pt 10pt,clip,width=1.6in]{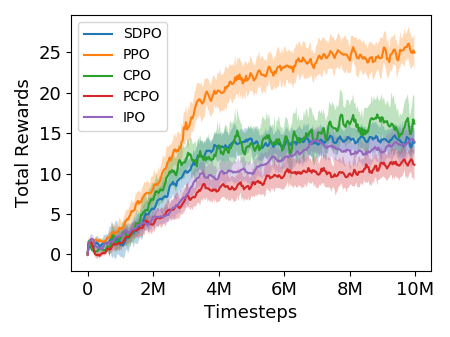}
\label{fig:pb2r}
\end{minipage}%
}%
\subfigure[constraint, PointButton-2]{
\begin{minipage}{0.24\columnwidth}
\centering
\includegraphics[trim=10pt 10pt 10pt 10pt,clip,width=1.6in]{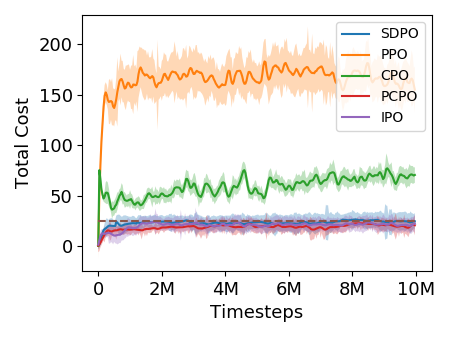}
\label{fig:pb2c}
\end{minipage}%
}
\subfigure[rewards, PointPush-1]{
\begin{minipage}{0.24\columnwidth}
\centering
\includegraphics[trim=10pt 10pt 10pt 10pt,clip,width=1.6in]{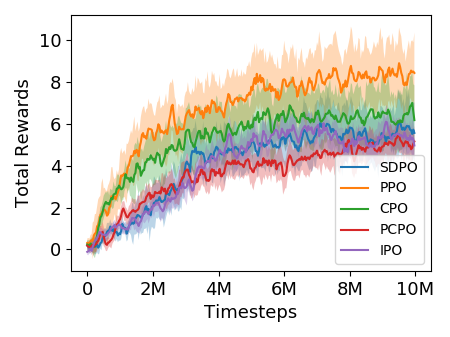}
\label{fig:pp1r}
\end{minipage}%
}%
\subfigure[constraint, PointPush-1]{
\begin{minipage}{0.24\columnwidth}
\centering
\includegraphics[trim=10pt 10pt 10pt 10pt,clip,width=1.6in]{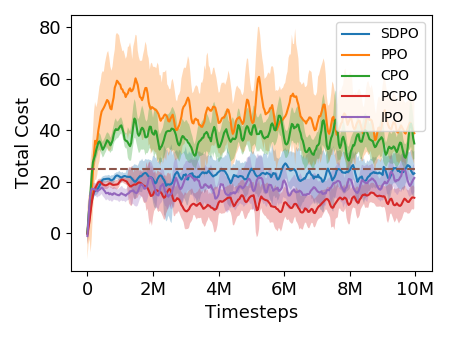}
\label{fig:pp1c}
\end{minipage}%
}
\subfigure[rewards, PointPush-2]{
\begin{minipage}{0.24\columnwidth}
\centering
\includegraphics[trim=10pt 10pt 10pt 10pt,clip,width=1.6in]{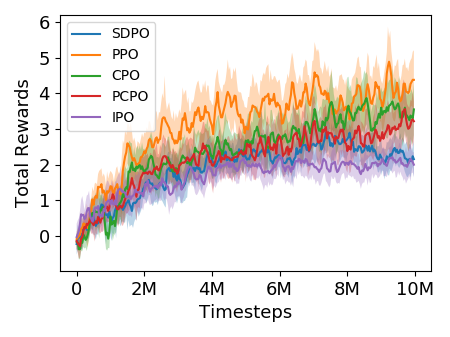}
\label{fig:pp2r}
\end{minipage}%
}%
\subfigure[constraint, PointPush-2]{
\begin{minipage}{0.24\columnwidth}
\centering
\includegraphics[trim=10pt 10pt 10pt 10pt,clip,width=1.6in]{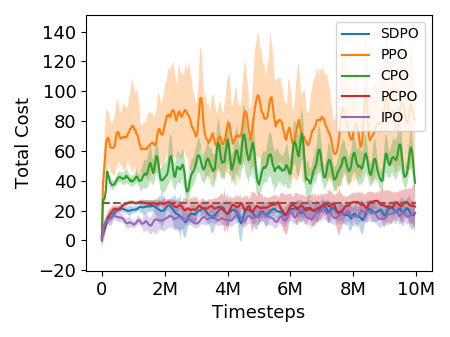}
\label{fig:pp2c}
\end{minipage}%
}
\caption{Average performance of the point agent over 10 runs of PPO, SDPO, PCPO and IPO under Safety-Gym. Both SDPO, PCPO and IPO converge to the level indicated by the dashed line.}
\label{fig:extra_point}
\end{figure}

\begin{figure}[h]
\centering
\subfigure[rewards, CarGoal-1]{
\begin{minipage}{0.24\columnwidth}
\centering
\includegraphics[trim=10pt 10pt 10pt 10pt,clip,width=1.6in]{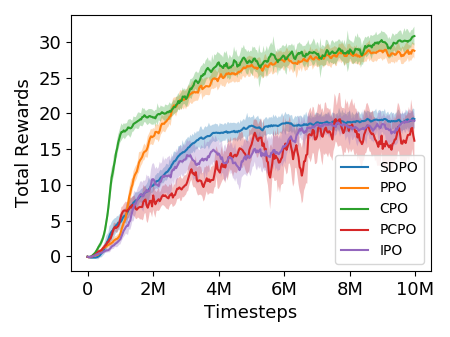}
\label{fig:cg1r}
\end{minipage}%
}%
\subfigure[constraint, CarGoal-1]{
\begin{minipage}{0.24\columnwidth}
\centering
\includegraphics[trim=10pt 10pt 10pt 10pt,clip,width=1.6in]{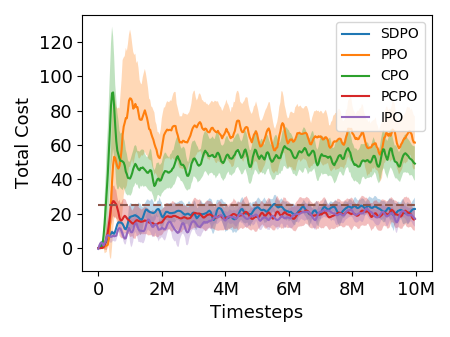}
\label{fig:cg1c}
\end{minipage}%
}
\subfigure[rewards, CarGoal-2]{
\begin{minipage}{0.24\columnwidth}
\centering
\includegraphics[trim=10pt 10pt 10pt 10pt,clip,width=1.6in]{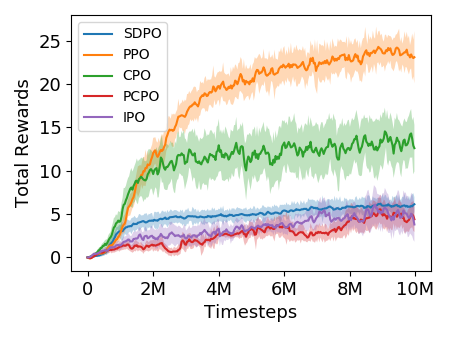}
\label{fig:cg2r}
\end{minipage}%
}%
\subfigure[constraint, CarGoal-2]{
\begin{minipage}{0.24\columnwidth}
\centering
\includegraphics[trim=10pt 10pt 10pt 10pt,clip,width=1.6in]{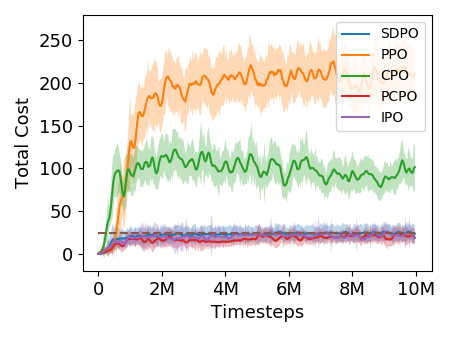}
\label{fig:cg2c}
\end{minipage}%
}
\subfigure[rewards, CarButton-1]{
\begin{minipage}{0.24\columnwidth}
\centering
\includegraphics[trim=10pt 10pt 10pt 10pt,clip,width=1.6in]{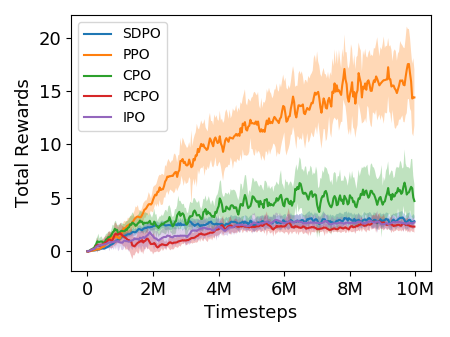}
\label{fig:cb1r}
\end{minipage}%
}%
\subfigure[constraint, CarButton-1]{
\begin{minipage}{0.24\columnwidth}
\centering
\includegraphics[trim=10pt 10pt 10pt 10pt,clip,width=1.6in]{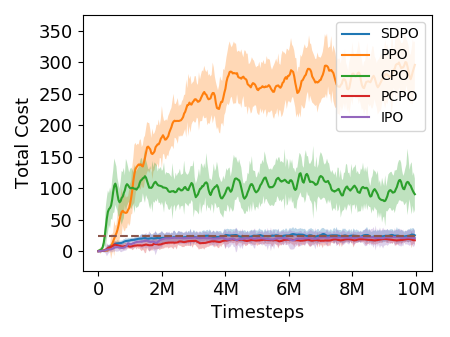}
\label{fig:cb1c}
\end{minipage}%
}
\subfigure[rewards, CarButton-2]{
\begin{minipage}{0.24\columnwidth}
\centering
\includegraphics[trim=10pt 10pt 10pt 10pt,clip,width=1.6in]{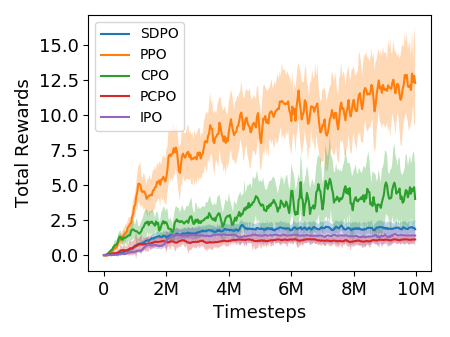}
\label{fig:cb2r}
\end{minipage}%
}%
\subfigure[constraint, CarButton-2]{
\begin{minipage}{0.24\columnwidth}
\centering
\includegraphics[trim=10pt 10pt 10pt 10pt,clip,width=1.6in]{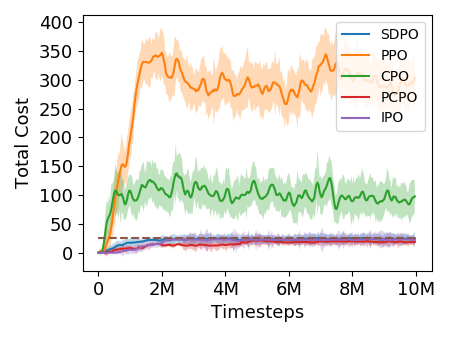}
\label{fig:cb2c}
\end{minipage}%
}
\subfigure[rewards, CarPush-1]{
\begin{minipage}{0.24\columnwidth}
\centering
\includegraphics[trim=10pt 10pt 10pt 10pt,clip,width=1.6in]{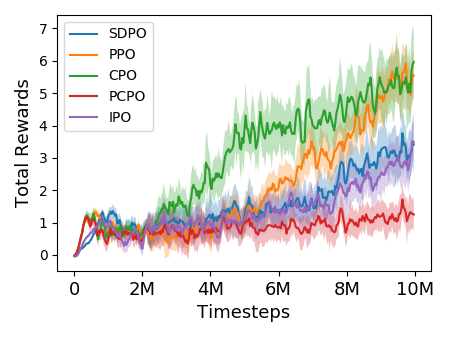}
\label{fig:cp1r}
\end{minipage}%
}%
\subfigure[constraint, CarPush-1]{
\begin{minipage}{0.24\columnwidth}
\centering
\includegraphics[trim=10pt 10pt 10pt 10pt,clip,width=1.6in]{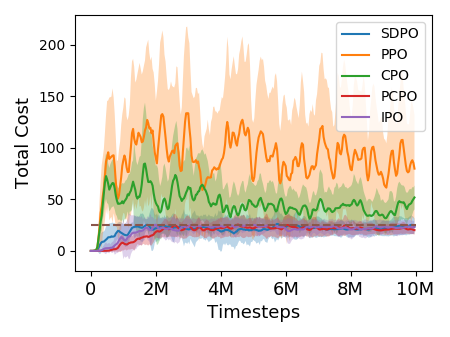}
\label{fig:cp1c}
\end{minipage}%
}
\subfigure[rewards, CarPush-2]{
\begin{minipage}{0.24\columnwidth}
\centering
\includegraphics[trim=10pt 10pt 10pt 10pt,clip,width=1.6in]{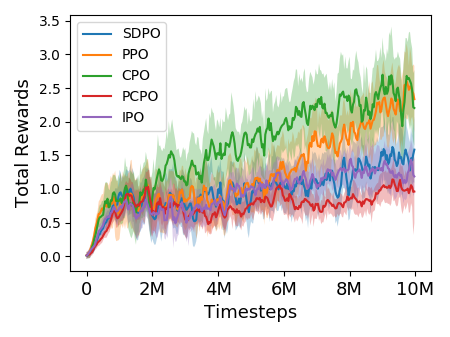}
\label{fig:cp2r}
\end{minipage}%
}%
\subfigure[constraint, CarPush-2]{
\begin{minipage}{0.24\columnwidth}
\centering
\includegraphics[trim=10pt 10pt 10pt 10pt,clip,width=1.6in]{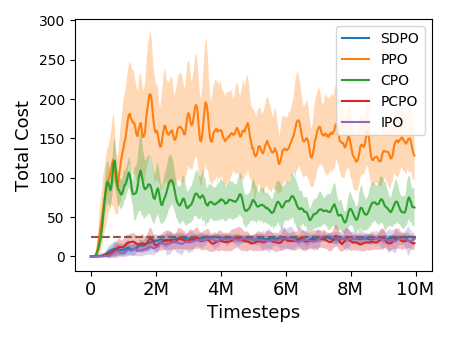}
\label{fig:cp2c}
\end{minipage}%
}
\caption{Average performance of the car agent over 10 runs of PPO, SDPO, PCPO and IPO under Safety-Gym. Both SDPO, PCPO and IPO converge to the level indicated by the dashed line.}
\label{fig:extra_car}
\end{figure}

\begin{figure}[h]
\centering
\subfigure[rewards, DoggoGoal-1]{
\begin{minipage}{0.24\columnwidth}
\centering
\includegraphics[trim=10pt 10pt 10pt 10pt,clip,width=1.6in]{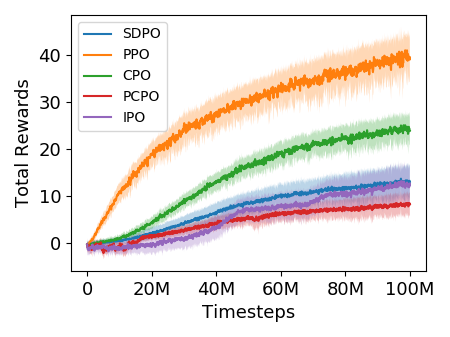}
\label{fig:dg1r}
\end{minipage}%
}%
\subfigure[constraint, DoggoGoal-1]{
\begin{minipage}{0.24\columnwidth}
\centering
\includegraphics[trim=10pt 10pt 10pt 10pt,clip,width=1.6in]{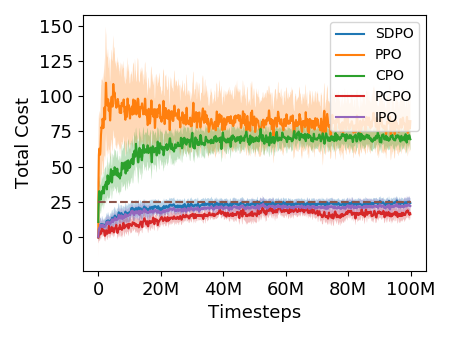}
\label{fig:dg1c}
\end{minipage}%
}
\subfigure[rewards, DoggoGoal-2]{
\begin{minipage}{0.24\columnwidth}
\centering
\includegraphics[trim=10pt 10pt 10pt 10pt,clip,width=1.6in]{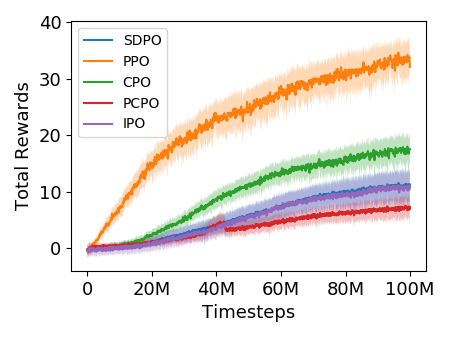}
\label{fig:dg2r}
\end{minipage}%
}%
\subfigure[constraint, DoggoGoal-2]{
\begin{minipage}{0.24\columnwidth}
\centering
\includegraphics[trim=10pt 10pt 10pt 10pt,clip,width=1.6in]{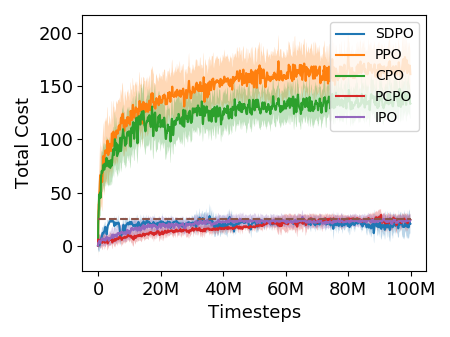}
\label{fig:dg2c}
\end{minipage}%
}
\subfigure[rewards, DoggoButton-1]{
\begin{minipage}{0.24\columnwidth}
\centering
\includegraphics[trim=10pt 10pt 10pt 10pt,clip,width=1.6in]{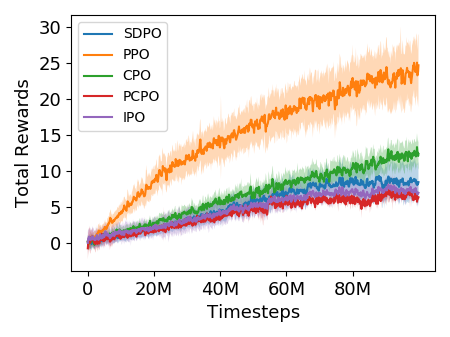}
\label{fig:db1r}
\end{minipage}%
}%
\subfigure[constraint, DoggoButton-1]{
\begin{minipage}{0.24\columnwidth}
\centering
\includegraphics[trim=10pt 10pt 10pt 10pt,clip,width=1.6in]{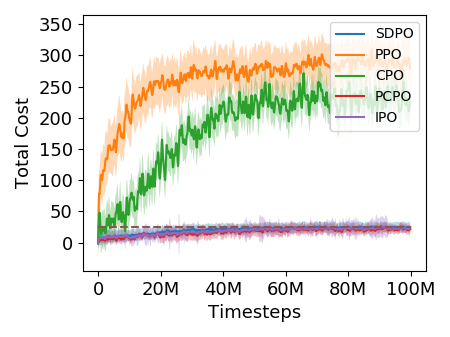}
\label{fig:db1c}
\end{minipage}%
}
\subfigure[rewards, DoggoButton-2]{
\begin{minipage}{0.24\columnwidth}
\centering
\includegraphics[trim=10pt 10pt 10pt 10pt,clip,width=1.6in]{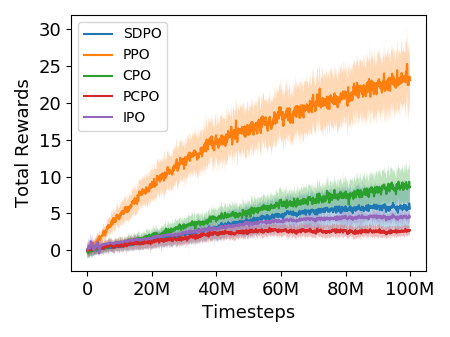}
\label{fig:db2r}
\end{minipage}%
}%
\subfigure[constraint, DoggoButton-2]{
\begin{minipage}{0.24\columnwidth}
\centering
\includegraphics[trim=10pt 10pt 10pt 10pt,clip,width=1.6in]{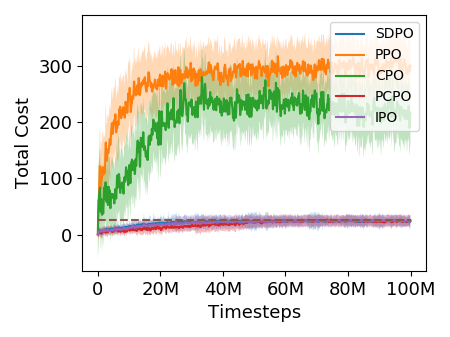}
\label{fig:db2c}
\end{minipage}%
}
\subfigure[rewards, DoggoPush-1]{
\begin{minipage}{0.24\columnwidth}
\centering
\includegraphics[trim=10pt 10pt 10pt 10pt,clip,width=1.6in]{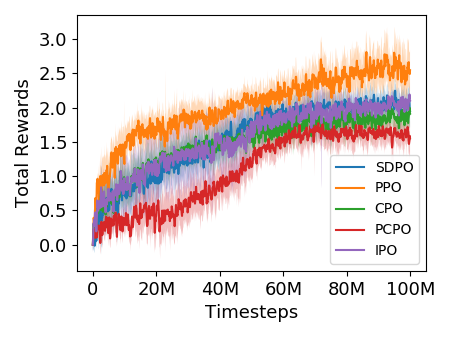}
\label{fig:dp1r}
\end{minipage}%
}%
\subfigure[constraint, DoggoPush-1]{
\begin{minipage}{0.24\columnwidth}
\centering
\includegraphics[trim=10pt 10pt 10pt 10pt,clip,width=1.6in]{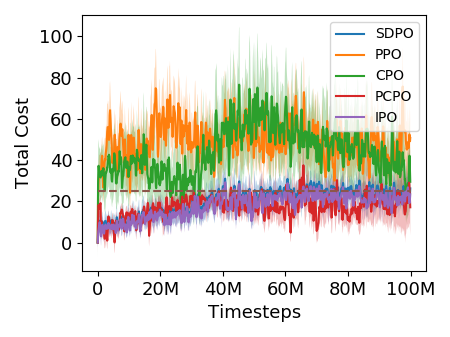}
\label{fig:dp1c}
\end{minipage}%
}
\subfigure[rewards, DoggoPush-2]{
\begin{minipage}{0.24\columnwidth}
\centering
\includegraphics[trim=10pt 10pt 10pt 10pt,clip,width=1.6in]{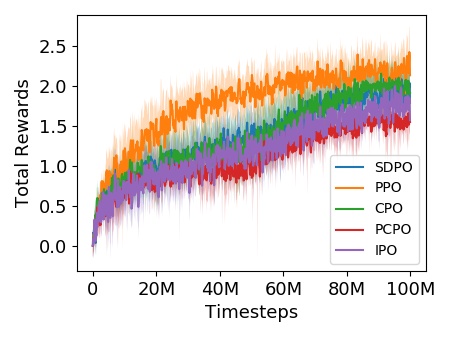}
\label{fig:dp2r}
\end{minipage}%
}%
\subfigure[constraint, DoggoPush-2]{
\begin{minipage}{0.24\columnwidth}
\centering
\includegraphics[trim=10pt 10pt 10pt 10pt,clip,width=1.6in]{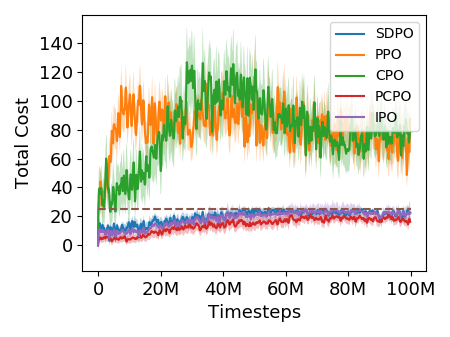}
\label{fig:dp2c}
\end{minipage}%
}
\caption{Average performance of the doggo agent over 10 runs of PPO, SDPO, PCPO and IPO under Safety-Gym. Both SDPO, PCPO and IPO converge to the level indicated by the dashed line.}
\label{fig:extra_doggo}
\end{figure}

From the results in \Cref{fig:extra_point,fig:extra_car,fig:extra_doggo}, SDPO, IPO and PCPO can explore the environment safely, but CPO may failed to learn a safe policy as the tasks go harder (i.e., with the car agent).
In experiments, when the CPO agents violate the constraints, Equation (14) in \citep{AchiamHeldTamarAbbeel17} failed to purely decrease the constraint value, and thus learn an unsafe policy.
For the other three agents, SDPO converges faster to  a slightly better policy than IPO and PCPO.
An interesting future work would be to extend PCPO to the distributional setting as well.

\end{document}